\documentclass[twocolumn,envcountsame]{svjour3}
\smartqed
\usepackage{amsmath,amssymb}
\usepackage{graphicx}
\usepackage{subfigure}
\usepackage{enumerate}
\usepackage{times}
\usepackage{hyperref}

\newcommand{\norm}[1]{\|#1\|}
\newcommand{\abs}[1]{|#1|}
\newcommand{\scp}[2]{\langle #1,#2\rangle}

\newcommand{\weakto}{\rightharpoonup}
\newcommand{\grad}{\nabla}

\newcommand{\extRR}{\mathbb{R}_\infty}
\newcommand{\R}{\mathbb{R}}

\newcommand{\multito}{\rightrightarrows}
\DeclareMathOperator{\id}{Id}

\title{An inertial forward-backward algorithm for monotone inclusions}
\author{D. Lorenz\and T. Pock} \institute{D. Lorenz\at Institute for
  Analysis and Algebra, TU Braunschweig, 38092 Braunschweig, Germany,
  \email{d.lorenz@tu-braunschweig.de}\and T. Pock\at Institute for
  Computer Graphics and Vision, Graz University of Technology,
  Inffeldgasse 16, 8010 Graz, Austria, and the Safety \& Security
  Department, AIT Austrian Institute of Technology GmbH,
  Donau-City-Stra{\ss}e 1, 1220 Vienna, Austria,
  \email{pock@icg.tugraz.at}}

\begin{document}
\maketitle

\begin{abstract}
  In this paper, we propose an inertial forward backward splitting
  algorithm to compute a zero of the sum of two monotone operators, with one
  of the two operators being co-coercive. The algorithm is inspired by
  the accelerated gradient method of Nesterov, but can be applied to a
  much larger class of problems including convex-concave saddle point problems and general monotone
  inclusions. We prove convergence of the algorithm in a Hilbert space
  setting and show that several recently proposed first-order methods
  can be obtained as special cases of the general algorithm.
  Numerical results show that the proposed algorithm converges faster
  than existing methods, while keeping the computational cost of each
  iteration basically unchanged.
\end{abstract}
\section{Introduction}

A fundamental problem is to find a zero of a maximal monotone
operator $T$ in a real Hilbert space $X$:
\begin{equation}\label{eq:monotone-inclusion}
\text{find}\ x \in X :\quad 0\in T(x).
\end{equation}
This problem includes, as special cases, variational inequality
problems, non-smooth convex optimization problems and convex-concave
saddle-point problems. Therefore this problem finds many important
applications in scientific fields such as image processing, computer
vision, machine learning and signal processing.

In case, $T=\nabla f$ is the gradient of a differentiable convex
function $f$, the most simple approach to
solve~\eqref{eq:monotone-inclusion} is to apply for each $k \geq 0$
the following recursion:
\begin{equation*}\label{eq:simple-forward}
  x^{k+1} = (\id - \lambda_k T)(x^k)\;,
\end{equation*}
where the operator $(\id - \lambda_k T)$ is the so-called forward
operator. Note, that the above scheme is nothing else than the
classical method of steepest descend and $\lambda_k >0$ is the step
size parameter that has to be chosen according to a rule that
guarantees convergence of the algorithm.

In case, $T$ is a general monotone operator, the classical algorithm
to solve~\eqref{eq:monotone-inclusion} is the proximal point algorithm
which can be traced back to the early works of Minty~\cite{Minty62}
and Martinet~\cite{Martinet70}. See also the thesis of
Eckstein~\cite{EcksteinPhD} for a detailed treatment of the
subject.

The proximal point algorithm generates a sequence $x^k$ according to
the recursion
\begin{equation}\label{eq:ppa}
x^{k+1} = (\id + \lambda_k T)^{-1}(x^k)\,,
\end{equation}
where $\lambda_k >0$ is a regularization parameter. The operator $(\id
+ \lambda_k T)^{-1}$ is the so-called resolvent operator, that has
been introduced by Moreau in~\cite{moreau65}. In the context of
algorithms, the resolvent operator is often referred to as the
backward operator. In the seminal
paper~\cite{Rockafellar1976_Proximal}, Rockafellar has shown that the
sequence $x^k$ generated by the proximal point algorithm converges
weakly to a point $x^*$ satisfying $0 \in T(x^*)$.

Unfortunately, in many interesting cases, the evaluation of the
resolvent operator is as difficult as solving the original
problem, which limits the practicability of the proximal point
algorithm in its plain form. To partly overcome this problem, it is
shown in~\cite{Rockafellar1976_Proximal}, that the algorithm still
converges when using inexact evaluations of the resolvent operator. In
fact, the evaluation errors have to satisfy a certain summability
condition which essentially means that the resolvent operators have to
be computed with increasing accuracy.  This is still somewhat
limiting, since in practice the errors of the resolvent operator are
often hard to control.

\subsection{Splitting methods}

In many problems, however, the operator $T$ can be written as the sum
of two maximal monotone operators, i.e. $T=A+B$, such that the
resolvent operators $(\id + \lambda A)^{-1}$ and $(\id + \lambda
B)^{-1}$, are much easier to compute than the full resolvent $(\id +
\lambda T)^{-1}$. Then, by combining the resolvents with respect to
$A$, and $B$ in a certain way, one might be able to mimic the effect
of the full proximal step based on $T$. The two most successful
instances that are based on combining forward and backward steps with
respect to $A$ and $B$, are the Peaceman-Rachford splitting
algorithm~\cite{PeacemanRachford55},
\[\label{eq:PR-splitting}
x^{k+1} = (\id + \lambda B)^{-1}(\id - \lambda A)(\id + \lambda
A)^{-1}(\id - \lambda B)(x^k)\;,
\]
and the Douglas-Rachford splitting algorithm~\cite{DouglasRachford56},
\[\label{eq:DR-splitting}
x^{k+1} = (\id + \lambda B)^{-1}[(\id + \lambda A)^{-1}(\id - \lambda
B) + \lambda B](x^k)\;.
\]
These splitting techniques have been originally proposed in the
context of linear operators and therefore cannot be applied to general
monotone operators. In~\cite{LionsMercier79}, Lions and Mer\-cier have
analyzed and further developed these splitting algorithms. Their idea
was to perform a change of variables \mbox{$x^k = (\id + \lambda
B)^{-1}(v^k)$}, such that the Peaceman-Rachford and Douglas-Rachford
splitting algorithms have a meaning even for $A$ and $B$ being
multivalued operators. Regarding convergence of the algorithms, the
Peaceman-Rachford algorithm still needs to assume that $B$ is
single-valued but the Douglas-Rachford algorithm converges even in the
general setting, where $A+B$ is just maximal monotone.

In~\cite{Eckstein1992}, Eckstein has pointed out that the
Douglas-Rach\-ford splitting algorithm can be re-written in the form
of~\eqref{eq:ppa}. Hence, it is basically a certain instance of the
proximal point algorithm. Moreover, Eckstein has shown that the
application of the Douglas-Rachford algorithm to the dual of a certain
structured convex optimization problem coincides with the so-called
alternating direction method of multipliers. It is remarkable, that
the Douglas-Rachford splitting algorithm and its variants have seen a
considerable renaissance in modern convex
optimization~\cite{GoldsteinOsher,Boyd_et_al}. The main reason for the
renewed interest lies in the fact that it is well suited for
distributed convex programming. This is an important aspect for
solving large scale convex optimization problems arising in recent
image processing and machine learning applications.

Another important line of splitting methods is given by the so-called
forward-backward splitting
technique~\cite{Goldstein64,LevitinPolyak1961,Bruck1977,LionsMercier79}. In
contrast to the more complicated splitting techniques discussed above,
the forward-backward scheme is based (as the name suggests) on the
recursive application of an explicit forward step with respect to $B$,
followed by an implicit backward step with respect
to $A$. The forward-backward algorithm is written as:
\begin{equation}\label{eq:fb-splitting}
x^{k+1}= (\id + \lambda_k A)^{-1}(\id - \lambda_k B)(x^k)
\end{equation}
In the most general setting, where both $A$ and $B$ are general
monotone operators, the convergence result is rather
weak~\cite{Passty79}, basically, $\lambda_k$ has to fulfill the same
step-size restrictions as unconstrained subgradient descend
schemes. However, if in addition $B$ is single valued and Lipschitz,
e.g. $B$ is the gradient of a smooth convex function, the situation
becomes much more beneficial. In fact, if $B$ is $L$-Lipschitz, and
$\lambda_k$ is chosen such that $\lambda_k < 2/L$, the forward
backward algorithm~\eqref{eq:fb-splitting} converges to a zero of
$T=A+B$~\cite{Gabay83,Tseng91}. Similar to the Douglas-Rachford
splitting algorithm, the forward-backward algorithm has seen a renewed
interest. It has been proposed and further improved in the context of
sparse signal recovery~\cite{DaubechieISTA,CombettesWajs}, image
processing~\cite{Raguet_et_al}, and machine
learning~\cite{DuchiSinger} applications.

\subsection{Inertial methods}

In~\cite{Polyak64}, Polyak introduced the so-called heavy ball method,
a two-step iterative method for minimizing a smooth convex function
$f$. The algorithm takes the following form:
\[\label{eq:heavy-ball}
\begin{cases}
  y^k = x^k + \alpha_k (x^k - x^{k-1})\\
  x^{k+1} = y^k - \lambda_k \nabla f(x^k)\;,
\end{cases}
\]
where $\alpha_k \in [0,1)$ is an extrapolation factor and $\lambda_k$
is again a step-size parameter that has to be chosen sufficiently
small.  The difference compared to a standard gradient method is that
in each iteration, the extrapolated point $y^k$ is used instead of
$x^k$. It is remarkable that this minor change greatly improves the
performance of the scheme.  In fact, its efficiency
estimate~\cite{Polyak64} on strongly convex functions is equivalent to
the known lower complexity bounds of first-order
methods~\cite{Nesterov} and hence the heavy-ball method resembles an
optimal method. The acceleration is explained by the fact that the new
iterate is given by taking a step which is a combination of the
direction $x^k - x^{k-1}$ and the current anti-gradient direction
$-\nabla f(x^k)$.

The heavy ball method can also be interpreted as an explicit finite
differences discretization of the time dynamical system
\[\label{eq:heavy-ball-cont}
\ddot x(t) + \alpha_1 \dot x(t) + \alpha_2 \nabla f(x(t)) = 0\,,
\]
where $\alpha_{1,2} > 0$ are free model parameters of the
equation. This equation is used to describe the motion of a heavy body
in a potential field $f$ and hence the system is coined the
heavy ball with friction dynamical system.

In~\cite{AlvarezAttouch2001}, Alvarez and Attouch translated the idea
of the heavy ball method to the setting of a general maximal monotone
operator using the framework of the proximal point
algorithm~\eqref{eq:ppa}. The resulting algorithm is called the
inertial proximal point algorithm and it is written as
\begin{equation}\label{eq:inertial-pp}
\begin{cases}
  y^k = x^k + \alpha_k (x^k - x^{k-1})\\
  x^{k+1} = (\id + \lambda_k T)^{-1}(y^k)\;,
\end{cases}
\end{equation}
It is shown that under certain conditions on $\alpha_k$ and
$\lambda_k$, the algorithm converges weakly to a zero of $T$. In fact,
the algorithm converges if $\lambda_k$ is non-decreasing and $\alpha_k
\in [0,1)$ is chosen such that
\begin{equation}\label{eq:law-alpha}
\sum_k \alpha_k \|x^k-x^{k-1}\|^2 < \infty\,,
\end{equation}
which can be achieved by choosing $\alpha_k$ with respect to a simple
on-line rule which ensures summability or in particular it is also
true for $\alpha_k < 1/3$.

In subsequent work~\cite{MoudafiOliny}, Moudafi and Oliny introduced
an additional single-valued and Lipschitz continuous operator $B$ into
the inertial proximal point algorithm:
\begin{equation}\label{eq:inertial-fb}
\begin{cases}
  y^k = x^k + \alpha_k (x^k - x^{k-1})\\
  x^{k+1} = (\id + \lambda_k A)^{-1}(y^k - \lambda_k B(x^k))\;,
\end{cases}
\end{equation}
It turns out that this algorithm converges as long as $\lambda_k <
2/L$, where $L$ is the Lipschitz constant of $B$ and the same
condition~\eqref{eq:law-alpha}, which is used to ensure convergence of
the inertial proximal point algorithm. Note that for $\alpha_k > 0$,
the algorithm does not take the form of a forward-backward splitting
algorithm, since $B$ is still evaluated at the point $x^k$.

In recent work, Pesquet and Pustelnik proposed a Douglas-Rachford type
parallel splitting method for finding the zero of the sum of an
arbitrary number maximal monotone operators. The method also includes
inertial forces~\cite{Pesquet_J_2012_j-pjo_parallel_ipo} which
numerically speeds up the convergence of the algorithm. Related
algorithms also including inertial forces have been proposed and
investigated in~\cite{Bot:arXiv1403.3330,BotCsetnekiADMM}.

\subsection{Optimal methods}

In a seminal paper~\cite{Nesterov83}, Nesterov proposed a modification
of the heavy ball method in order to improve the convergence rate on
smooth convex functions. While the heavy ball method evaluates the
gradient in each iterate at the point $x^k$, the idea of Nesterov was
to use the extrapolated point $y^k$ also for evaluating the
gradient. Additionally, the extrapolation parameter $\alpha_k$ is
computed according to some special law that allows to prove optimal
convergence rates of this scheme. The scheme is given by:
\begin{equation}\label{eq:nesterov-optimal}
\begin{cases}
  y^{k} = x^{k} + \alpha_k(x^{k} - x^{k-1})\\
  x^{k+1} = y^{k} - \lambda_k \nabla f(y^k)\,,
\end{cases}
\end{equation}
where $\lambda_k = 1/L$, There are several choices to define an
optimal sequence
$\{\alpha_k\}$~\cite{Nesterov83,Nesterov,BeckTeboulle2008,tseng2008accelerated}.
In~\cite{Nesterov}, it has been shown that the efficiency estimate of
the above scheme is up to some constant factor equivalent to the lower
complexity bounds of first-order methods for the class of
$\mu$-strongly convex functions, $\mu \geq 0$, with $L$-Lipschitz
gradient.

In~\cite{Gueler1991}, G\"uler has translated Nesterov's idea to the
general setting of the proximal point algorithm, with the restriction
that the operator $T$ is the subdifferential of a convex
function. Inexact versions of this algorithm have been proposed and
studied in~\cite{Villa_etal}. In~\cite{BeckTeboulle2008}, Beck and
Teboulle have proposed the so-called fast iterative shrinkage
thresholding algorithm (FISTA), that combines in a clever way the
ideas of Nesterov and G\"uler within the forward-backward splitting
framework. The algorithm features the same optimal convergence rate as
Nesterov's method but it can be applied also in the presence of an
additional but simple (with easy to compute proximal map) non-smooth
convex function. The FISTA algorithm can be applied to a variety of
practical problem arising in sparse signal recovery, image processing
and machine learning and hence has become a standard
algorithm. Related algorithms with similar properties have been
independently proposed by Nesterov
in~\cite{NesterovSNS,NesterovComposite}.

\subsection{Content}

In this paper we propose a modification of the forward-back\-ward
splitting algorithm~\eqref{eq:fb-splitting} to solve monotone
inclusions.  Our method is inspired by the inertial forward-backward
splitting method~\eqref{eq:inertial-fb}, but differs from this method
in two regards. First, the operator $B$ is evaluated at the inertial
extrapolate $y^k$ which is inspired by Nesterov's optimal gradient
method~\eqref{eq:nesterov-optimal}. In addition, we consider a
symmetric positive definite map $M$, which can be interpreted as a
preconditioner or variable metric and is inspired by recently work on
primal dual
algorithms~\cite{CP2010,Esseretal10,pock2011_precond,HeYuan2012} and
forward backward
splitting~\cite{CombettesVu,ChozenouxPesquetRepetti,ChenRockafellar97}. These
changes allow us to define a new ``meta-algorithm'', that includes, as
special cases for example several convex optimization algorithms that
have recently attracted a lot of attention in the imaging, signal
processing and machine learning communities.

In section~\ref{sec:fb-monotone} we will define the proposed algorithm
and prove the general convergence in a Hilbert space setting.  In
section~\ref{sec:pdfb} we will apply the proposed algorithm to a class
of convex-concave saddle-point problems and will show how several
known algorithms can be recovered from the proposed
``meta-algorithm''. In section~\ref{sec:numerics}, we will apply the
proposed algorithm to image processing problems including, image
restoration and image deconvolution. In the last section, we will give
some concluding remarks.

\section{Proposed algorithm}
\label{sec:fb-monotone}

We consider the problem of finding a point $x^*$ in a Hilbert space $X$
such that
\begin{equation}\label{eq:ab-inclusion}
  0 \in (A+B)(x^*)\,,
\end{equation}
where $A,B$ are maximal monotone operators. We additionally assume
that the operator $B$ is single-valued and co-coercive with respect to
the solution set $S:=(A+B)^{-1}(0)$ and a linear, selfadjoint and positive definite map $L$, i.e. for all $x\in X$, $y\in S$
\begin{equation}
  \label{eq:cocoercive}
  \scp{B(x)-B(y)}{x-y} \geq \norm{B(x)-B(y)}_{L^{-1}}^2
\end{equation}
where, as
usual, we denote $\norm{x}_{L^{-1}}^2 = \scp{L^{-1}x}{x}$. Note that
in the most simple case where $L=l\id$, $l > 0$, the operator $B$
is $1/l$ co-coercive and hence $l$-Lipschitz. However, we will later
see that in some cases, it makes sense to consider more general $L$.

The algorithm we propose in this paper is a basically a modification
of the forward-backward splitting algorithm~\eqref{eq:fb-splitting}.
The scheme is as follows:
\begin{equation}\label{eq:inclusion-iterate}
\begin{cases}
  y^k = x^k + \alpha_k (x^k - x^{k-1})\\
  x^{k+1} = (\id + \lambda_k M^{-1} A)^{-1}(\id - \lambda_k M^{-1} B)(y^k)\;,
\end{cases}
\end{equation}
where $\alpha_k \in [0,1)$ is an extrapolation factor, $\lambda_k$
  is a step-size parameter and $M$ is a linear selfadjoint and positive
  definite map that can be used as a preconditioner for the algorithm
  (cf. Section~\ref{sec:precond}).  Note that the iteration can be
  equivalently expressed as
\begin{equation}\label{eq:inclusion-iterate-var}
\begin{cases}
  y^k = x^k + \alpha_k (x^k - x^{k-1})\\
  x^{k+1} = (M + \lambda_k A)^{-1}(M - \lambda_k B)(y^k)\;,
\end{cases}
\end{equation}
Observe that~\eqref{eq:inclusion-iterate}
(resp.~\eqref{eq:inclusion-iterate-var}) differs from the inertial
forward-backward algorithm of Moudafi and Oliny insofar that we also
evaluate the operator $B$ at the inertial extrapolate $y^k$. This
allows us to rewrite the algorithm in the form of the standard
forward-backward algorithm~\eqref{eq:fb-splitting}.

In the following Theorem, we analyze the basic convergence properties
of the proposed algorithm.
\begin{theorem}
  \label{thm:intertial_fb_convergence}
  Let $X$ be a real Hilbert space and $A,B:X \multito X$ be maximally
  monotone operators. Further assume that $M,L:X\to X$ are linear,
  bounded, selfadjoint and positive definite maps and that $B$ is
  single valued and co-coercive w.r.t. $L^{-1}$
  (cf.~(\ref{eq:cocoercive})). Moreover, let $\lambda_k>0$,
  $\alpha<1$, $\alpha_k\in[0,\alpha]$, $x^0=x^{-1}\in X$ and let the
  sequences $x^k$ and $y^k$ be defined by~(\ref{eq:inclusion-iterate})
  (or~\eqref{eq:inclusion-iterate-var}).  If
  \begin{enumerate}[(i)]
  \item $S_k = M - \tfrac{\lambda_k}{2}L$ is positive definite for all $k$ and
    \label{it:assumption-Sk-pos-def}
  \item $\sum_{k=1}^\infty \alpha_k\norm{x^k-x^{k-1}}_M^2<\infty$  \label{it:assumpion-convergence}
  \end{enumerate}
  then $x^k$ converges weakly to a solution of the inclusion $0\in
  (A+B)(x)$.
\end{theorem}
\begin{proof}
Denote by $x^*$ a zero of $A+B$. From~\eqref{eq:ab-inclusion},
it holds that
\[
-B(x^*) \in A(x^*)\,.
\]
Furthermore, the second line in~\eqref{eq:inclusion-iterate-var} can be
equivalently expressed as
\[
M(y^k-x^{k+1}) - \lambda_kB(y^k) \in\lambda_kA(x^{k+1})\,.
\]
For convenience, we define for any symmetric positive definite $M$,
\begin{align*}
  \label{eq:phik}
  \phi_M^k &=  \tfrac12\norm{x^k-x^*}_M^2 = \tfrac12\scp{M(x^k-x^*)}{x^k-x^*}\,,\\
  \Delta_{M}^k &= \tfrac12\norm{x^{k}-x^{k-1}}_{M}^2 = \tfrac12\scp{M(x^k-x^{k-1})}{x^k-x^{k-1}}\,\\
  \Gamma_{M}^k &= \tfrac12\norm{x^{k+1}-y^{k}}_{M}^2 = \tfrac12\scp{M(x^{k+1}-y^{k})}{x^{k+1}-y^{k}}\,.
\end{align*}
From the well-known identity
\begin{equation}\label{eq:three-point-phytagoras}
\scp{a-b}{a-c}_{M} = \tfrac12\norm{a-b}_{M}^2 +
\tfrac12\norm{a-c}_{M}^2 - \tfrac12\norm{b-c}_{M}^2
\end{equation}
we have by using the definition of the inertial extrapolate $y_k$ that
\begin{equation}\label{eq:phik-phik-1-0}
  \begin{split}
    \phi_{M}^k - \phi_{M}^{k+1} = \Delta_{M}^{k+1} +
    \scp{y^k-x^{k+1}}{x^{k+1}-x^*}_{M}\\
    \qquad- \alpha_k
    \scp{x^k-x^{k-1}}{x^{k+1}-x^*}_{M}\;.
  \end{split}
\end{equation}
Then, by using the monotonicity of $A$ we deduce that
\begin{align*}
  \scp{\lambda_kA(x^{k+1})-\lambda_kA(x^*)}{x^{k+1}-x^*} &\geq 0\\
  \scp{M(y^k-x^{k+1}) - \lambda_kB(y^k) + \lambda_kB(x^*)}{x^{k+1}-x^*} &\geq 0
\end{align*}
and
\[
\begin{split}
  &\scp{y^k-x^{k+1}}{x^{k+1}-x^*}_{M}\\
  &\qquad+ \lambda_k\scp{B(x^*) -
    B(y^k)}{x^{k+1}-x^*} \geq 0\,.
\end{split}
\]
Combining with~\eqref{eq:phik-phik-1-0}, we obtain
\begin{equation}\label{eq:phik-phik1}  
  \begin{split}
    \phi_{M}^k - \phi_{M}^{k+1} \geq \Delta_{M}^{k+1}
    +\lambda_k\scp{B(y^k) - B(x^*)}{x^{k+1}-x^*}\\
    \qquad-\alpha_k\scp{x^k-x^{k-1}}{x^{k+1}-x^*}_{M}\,.
  \end{split}
\end{equation}
From the co-coercivity property of $B$ we have that
\begin{align*}
  &\scp{B(y^k)-B(x^*)}{x^{k+1}-x^*}\\
  & = \scp{B(y^k)-B(x^*)}{x^{k+1}-y^k + y^k-x^*}\\
  &\geq \norm{B(y^k)-B(x^*)}_{L^{-1}}^2 + \scp{B(y^k)-B(x^*)}{x^{k+1}-y^k}\\
  & \geq \norm{B(y^k)-B(x^*)}_{L^{-1}}^2 - \norm{B(y^k)-B(x^*)}_{L^{-1}}^2 - \tfrac12\Gamma_L^k\\
  & = -\tfrac12\Gamma_L^k
\end{align*}
Substituting back into~\eqref{eq:phik-phik1} we arrive at
\begin{align*}
  \phi_{M}^k - \phi_{M}^{k+1} & \geq \Delta_{M}^{k+1} - \tfrac{\lambda_k}{2}\Gamma_L^k - \alpha_k\scp{x^k-x^{k-1}}{x^{k+1}-x^*}_{M}
\end{align*}
Invoking again~\eqref{eq:three-point-phytagoras}, it follows that
\begin{equation}\label{eq:phik-phik-1-2}
  \begin{split}
    &\phi_{M}^{k+1} - \phi_{M}^k - \alpha_k\left(\phi_{M}^k-\phi_{M}^{k-1}\right)\\
    &\leq - \Delta_{M}^{k+1} + \tfrac{\lambda_k}{2}\Gamma_L^k\\
    &\quad + \alpha_k\left(\Delta_{M}^{k} + \scp{x^k-x^{k-1}}{x^{k+1}-x^k}_{M}\right)\\
    &= - \Gamma_M^k + \tfrac{\lambda_k}{2}\Gamma_L^k +
    (\alpha_k+\alpha_k^2)\Delta_{M}^{k}\,.
  \end{split}
\end{equation}
The rest of the proof closely follows the proof of Theorem 2.1
in~\cite{AlvarezAttouch2001}. By the definition of $S_k$ and using
$(\alpha_k+\alpha_k^2)/2 \leq \alpha_k$, we have
\begin{align}
  \label{eq:phik-phik-1-3}
  \phi_{M}^{k+1} - \phi_{M}^k - \alpha_k(\phi_{M}^k-\phi_{M}^{k-1}) \leq - \Gamma_{S_k}^k + 2\alpha_k\Delta_M^k\,.
\end{align}
By assumption~(\ref{it:assumption-Sk-pos-def}), the first term is non-positive and
since $\alpha_k \geq 0$, the second term is non-negative.

Now, defining $\theta^k = \max(0, \phi_{M}^k-\phi_{M}^{k-1})$ and setting
\[
\delta^k = 2\alpha_k\Delta_M^k = \alpha_k\norm{x^{k}-x^{k-1}}_{M}^2\,,
\]
we obtain
\begin{align*}
\theta^{k+1}  \leq \alpha_k\theta^k + \delta^k \leq \alpha \theta^k + \delta^k
\end{align*}
Applying this inequality recursively, one obtaines a geometric series
of the form
\[
\theta^{k+1} \leq \alpha^k\theta^1 + \sum_{i=0}^{k-1}\alpha^i \delta^{k-i}
\]
Summing this inequality from $k=0,\dots,\infty$, one has
\[
\sum_{k=0}^\infty \theta^{k+1} \leq \frac{1}{1-\alpha}\left( \theta^1 +\sum_{k=1}^{\infty} \delta^k \right)
\]
Note that the series on the right hand side converges by assumption~(\ref{it:assumpion-convergence}).

Now we set $t^k = \phi_M^k - \sum_{i=1}^k\theta^k$
and since $\phi_M^k\geq 0$ and $\sum_{i=1}^k\theta_i$ is bounded independently of $k$, we
see that $t^k$ is bounded from below. On the other hand,
\begin{align*}
  t^{k+1} & = \phi_M^{k+1} - \theta^{k+1} - \sum_{i=1}^k\theta^i\\
  & \leq \phi_M^{k+1} - \phi_M^{k+1}+\phi_M^k - \sum_{i=1}^k\theta^i = t^k
\end{align*}
and hence, $t^k$ is also non-decreasing, thus convergent. This implies
that $\phi_M^k$ is convergent and especially that $\theta^k\to 0$.

From~\eqref{eq:phik-phik-1-3} we get
\begin{align*}
\tfrac12\norm{x^{k+1}-y^k}_{S_k}^2  & \leq -\theta^{k+1} -\alpha \theta^k + \delta^k\\
\tfrac12\norm{x^{k+1}-x^k-\alpha_k(x^k-x^{k-1})}_{S_k}^2  & \leq -\theta^{k+1} -\alpha \theta^k + \delta^k
\end{align*}
Since $\delta^k$ is summable it follows that
$\norm{x^k-x^{k-1}}_{S_k}\to 0$ and hence
\[
\lim_{k\rightarrow \infty}
\norm{x^{k+1}-x^k-\alpha_k(x^k-x^{k-1})}_{S_k} = 0\,.
\]
We already know that $x^k$ is bounded hence, there is a convergent
subsequence $x_\nu\weakto \bar x$. Then we also get that $y_\nu =
(1+\alpha_\nu)x_\nu - \alpha_\nu x_{\nu-1}\weakto \bar x$.
Now we get from~\eqref{eq:inclusion-iterate} that
\[
x^{\nu} = (\id + \lambda_\nu M^{-1}A)^{-1}(y^\nu-\lambda_\nu M^{-1}B(y^\nu))
\]
and pass to the limit (extracting another subsequence such that
$\lambda_\nu\to\bar\lambda$ if necessary) to obtain
\[
\bar x = (\id +
  \bar\lambda M^{-1}A)^{-1}(\bar x-\bar\lambda M^{-1}B(\bar x))
\]
which is equivalent to
\[
-B(\bar x)\in A(\bar x)
\]
which in turn shows that $\bar x$ is a solution. Opial's
Theorem~\cite{Opial} concludes the proof.\qed
\end{proof}

Next, we address the question whether the sequence $\{\alpha_k\}$ can
be chosen a-priori such that the algorithm is guaranteed to converge.
Indeed, in case of the inertial proximal point
algorithm~\eqref{eq:inertial-pp}, it has already been shown
in~\cite{AlvarezAttouch2001} that convergence is ensured if $\{\alpha_k\}$
is a nondecreasing sequence in $[0,\alpha]$ with $\alpha < 1/3$. The
next theorem presents a related result for the proposed algorithm.
\begin{theorem}
  \label{thm:intertial_fb_alphas}
  In addition to the conditions to
  Theorem~\ref{thm:intertial_fb_convergence} assume that
  $\{\lambda_k\}$ and $\{\alpha_k\}$ are nondecreasing sequences and
  that there exists a $\varepsilon > 0$ such that for all $\alpha_k$
  \begin{equation}\label{eq:cond-alpha}
  R_k = (1-3\alpha_{k})M - (1-\alpha_{k})^2\tfrac{\lambda_{k}}{2} L \geq \varepsilon M\,.
  \end{equation}
  Then $x^k$ converges weakly to a solution of the inclusion $0\in
  (A+B)(x^*)$.
\end{theorem}
\begin{proof}
  The proof of this result is an adaption of the proof of Proposition
  2.1 in~\cite{AlvarezAttouch2001}. From the last estimate in~\eqref{eq:phik-phik-1-2}
  and using the definition of $y^k$ in~\eqref{eq:inertial-fb} it
  follows that
  \begin{align*}
    & \phi_{M}^{k+1} - \phi_{M}^k - \alpha_k(\phi_{M}^k-\phi_{M}^{k-1})\\
    & \leq - \Gamma_{S_k}^k + \alpha_k(1+\alpha_k)\Delta_{M}^{k}\\
    & \leq - \Delta^{k+1}_{S_k} - \alpha_k^2\Delta^{k}_{S_k} + \alpha_k\scp{x^{k+1}-x^k}{x^k-x^{k-1}}_{S_k}\\
    & \qquad+ (\alpha_k+\alpha_k^2)\Delta_{M}^{k}\\
    & \leq (\alpha_k-1) \Delta^{k+1}_{S_k} + (\alpha_k-\alpha_k^2)\Delta^{k}_{S_k} + (\alpha_k+\alpha_k^2)\Delta_{M}^{k}\\
    & \leq (\alpha_k-1) \Delta^{k+1}_{S_k} + \alpha_k \Delta^{k}_{T_k}\,,
  \end{align*}
  where $T_k = 2M - \tfrac{(1-\alpha_k)\lambda_k}{2}L$.

  We define $\mu^k = \phi_{M}^k - \alpha_k \phi_{M}^{k-1} + \alpha_k
  \Delta^{k}_{T_k}$ and since $\alpha_{k+1} \geq \alpha_k$ and using
  the above inequality,
  \begin{align*}
  & \mu^{k+1} - \mu^{k}\\
  & = \phi_{M}^{k+1} - \alpha_{k+1} \phi_{M}^{k} +
        \alpha_{k+1} \Delta_{T_{k+1}}^{k+1} - \phi_{M}^k + \alpha_k \phi_{M}^{k-1} - \alpha_k\Delta_{T_k}^k\\
        & \leq \phi_{M}^{k+1}- \phi_{M}^k - \alpha_{k} (\phi_{M}^{k}-\phi_{M}^{k-1}) +
        \alpha_{k+1} \Delta_{T_{k+1}}^{k+1} - \alpha_k\Delta_{T_k}^k\\
        & \leq (\alpha_k-1) \Delta_{S_k}^{k+1} + \alpha_{k+1} \Delta_{T_{k+1}}^{k+1}\,.
  \end{align*}
  Then, we obtain since $\alpha_{k+1} \geq \alpha_k$
  \begin{align*}
    & \mu^{k+1} - \mu^{k}\\
    & \leq \tfrac{1}{2}\scp{\left((\alpha_k-1) S_k + \alpha_{k+1}T_k\right)(x^{k+1}-x^k)}{x^{k+1}-x^k}\\
    & \leq \tfrac{1}{2}\scp{\left((\alpha_{k+1}-1) S_k + \alpha_{k+1}T_k\right)(x^{k+1}-x^k)}{x^{k+1}-x^k}\;.
  \end{align*}
  Now using $\alpha_{k+1} \geq \alpha_k$ and $\lambda_{k+1} \geq
  \lambda_k$ we obtain
  \[
  \begin{split}
    & (\alpha_{k+1}-1) S_k + \alpha_{k+1}T_k\\
    & \leq ( 3\alpha_{k+1}-1)M +
    (1-\alpha_{k+1})^2\tfrac{\lambda_{k+1}}{2} L= R_k
  \end{split}\]
  which finally gives
  \begin{equation}\label{eq:mu-estmiate}
    \mu^{k+1} - \mu^{k} \leq -\Delta_{R_k}^k.
  \end{equation}
  Observe that by assumption~\eqref{eq:cond-alpha}, the sequence
  $\{\mu_k\}$ is non-increasing and hence
  \[
  \phi_M^k - \alpha\phi_M^{k-1} \leq \mu^k \leq \mu^1\,.
  \]
  It follows that
  \[
  \phi_M^k \leq \alpha^k \phi^0 + \mu^1\sum_{i=0}^{k-1}\alpha^i \leq \alpha^k \phi^0 + \frac{\mu^1}{1-\alpha}
  \]
  On the other hand, we have by summing up~\eqref{eq:mu-estmiate} from
  $i=1$ to $k$,
  \[
  \mu^{k+1}-\mu^1 \leq - \sum_{i=1}^k \Delta_{R_i}^i\,.
  \]
  Combining these two estimates it follows that
  \[
  \sum_{i=1}^k \Delta_{R_i}^i \leq \mu^1 - \mu^{k+1} \leq \mu^1 +
  \alpha \phi_M^k \leq \alpha^{k+1} \phi^0 + \frac{\mu^1}{1-\alpha}\,.
  \]
  Since $R_k \geq \varepsilon M$, it follows that
  \[
  \sum_{k=1}^\infty \Delta_{M}^k < \infty\,,
  \]
  which especially shows (ii) in
  Theorem~\ref{thm:intertial_fb_convergence}. The weak convergence of
  the $x^k$ now follows from
  Theorem~\ref{thm:intertial_fb_convergence}.\qed
\end{proof}

\begin{remark}\label{remark1}
  In case, $M=m\id$, $L=l\id$, $\lambda_k \equiv \lambda$ and defining
  the normalized step size $\gamma=\frac{ l \lambda}{m} \in(0,2)$,
  assertion~\eqref{eq:cond-alpha} reduces to
  \[
  1 - 3 \alpha_{k} - \varepsilon - \frac{(1-\alpha_{k})^2}{2}\gamma \geq 0\,.
  \]
  It easily follows that for any $\varepsilon \in (0,
  (9-4\gamma)/(2\gamma))$, the algorithm converges, if the sequence
  $\{\alpha_k\}$ is non-decreasing with $0\leq\alpha_k\leq\alpha(\gamma)$, where
  \[
  \alpha(\gamma) = 1 + \frac{\sqrt{9 - 4\gamma - 2\varepsilon\gamma} - 3}{\gamma}\;.
  \]
  See Figure~\ref{fig:alpha} for a plot of $\alpha(\gamma)$ using
  $\varepsilon=10^{-6}$.
\end{remark}

\begin{figure}
  \centering
  \includegraphics[width=0.3\textwidth]{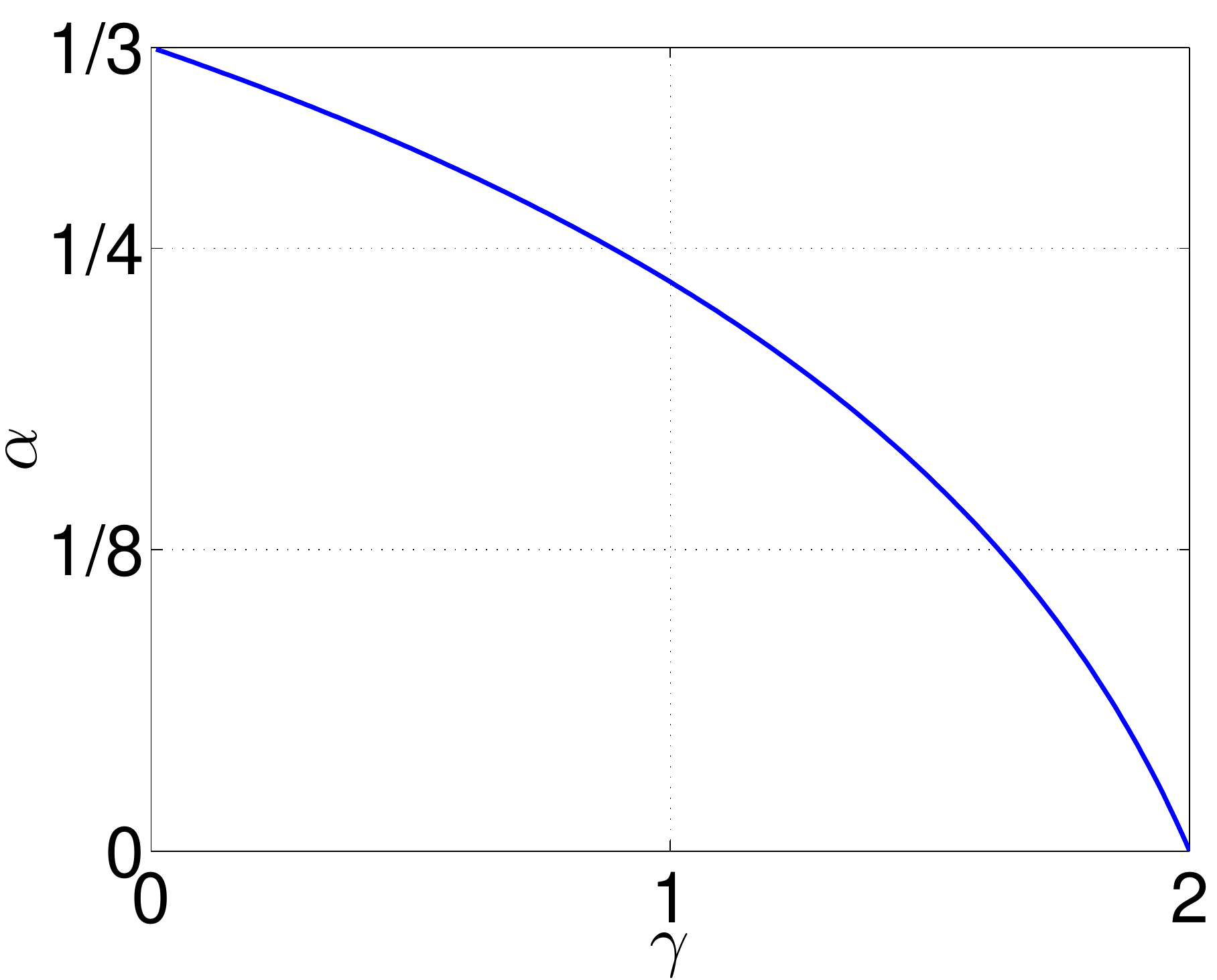}
  \caption{Upper bound on the extrapolation factor $\alpha$ in
    dependence on $\gamma$.}
    \label{fig:alpha}
\end{figure}

\begin{remark}\label{remark:implicit}
  Let us consider a ``fully-implicit'' variant of the
  scheme~\eqref{eq:inclusion-iterate}, which is given by
  \[
  \begin{cases}
    y^k = x^k + \alpha_k(x^k-x^{k-1})\,,\\
    x^{k+1} = (\id + \lambda_k \overline M^{\,-1} (A+B))^{-1}(y^k)\,,
  \end{cases}
  \]
  where $\overline M$ is again a linear, selfadjoint and positive
  definite map.  In fact this algorithm, is an inertial proximal point
  algorithm, in the $\overline M$ metric, whose convergence properties
  have been studied in~\cite{AlvarezAttouch2001}. This algorithm has less
  stringent convergence properties compared to the algorithm proposed
  in this paper, but its application to practical problems is limited
  since the resolvent with respect to $A+B$ can be complicated.

  Interestingly, if the operator $B$ is a linear, selfadjoint and
  positive semi-definite map, the above fully-implicit scheme can be
  significantly simplified. In fact, using $\lambda_k\equiv\lambda$
  and setting $\overline M = M - \lambda B$, where $\lambda$ is chosen such
  that $\overline M > 0$, it turns out that the fully implicit scheme in
  the $\overline M$ metric is equivalent to our proposed inertial
  forward-backward splitting algorithm~\eqref{eq:inclusion-iterate} in
  the $M$ metric, which only requires to compute the resolvent with
  respect to $A$.

  According to Theorem 2.1 and Proposition 2.1
  in~\cite{AlvarezAttouch2001}, condition (i) of Theorem 1 can be
  replaced by the simpler condition $M - \lambda B > 0$ and
  convergence of the algorithm is guaranteed for $\{\alpha_k\}$
  non-decreasing in $[0,\alpha]$ with $\alpha < 1/3$.
\end{remark}

\section{Application to convex-concave saddle-point problems}
\label{sec:pdfb}

Recently, so-called primal-dual splitting techniques have been
proposed which are motivated by the need to solve large-scale
non-smooth convex optimization problems in image
processing~\cite{CP2010,Esseretal10,pock2011_precond,HeYuan2012,Condat2013,CombettesPesquet12,Vu2013}. These algorithms can be
applied if the structure of the problem allows to rewrite
it as certain convex-concave saddle-point problems.

Now let $X$ and $Y$ be two Hilbert spaces and consider the saddle
point problem
\begin{equation}
  \label{eq:saddle-point-problem}
  \min_{x\in X}\max_{y\in Y} G(x) + Q(x) + \scp{Kx}{y} - F^*(y) - P^*(y)
\end{equation}
with convex $G,Q:X \to \extRR$, $F^*,P^*:Y\to\extRR$, $K:X\to Y$
linear and bounded and $Q,P^*$ differentiable with Lipschitz gradient
(with respective Lipschitz constants $L_Q$, $L_P$).

We define the monotone operators $A,B$ on $X\times Y$ as
\[
A =
\begin{bmatrix}
  \partial G & K^*\\
  -K & \partial F^*
\end{bmatrix},\qquad
B =
\begin{bmatrix}
  \grad Q & 0\\
  0 & \grad P^*
\end{bmatrix}
\]
and observe that the optimality system of the saddle point problem can
be written as
\[
0 \in (A + B)
\begin{bmatrix}
  x\\
  y
\end{bmatrix}.
\]
This setup fits into our general framework of
section~\ref{sec:fb-monotone}.

The standard splitting iterations (\ref{eq:fb-splitting}) and
(\ref{eq:inertial-fb}) would not be applicable in general since the
evaluation of the proximal mapping $(\id + \lambda A)^{-1}$ may be
prohibitively expensive in this case.  However, if we consider the
preconditioned iteration~(\ref{eq:inclusion-iterate}) with an
appropriate mapping $M$ the iteration becomes feasible. The idea is,
to choose $M$ such that one of the off-diagonal blocks in $A$ cancel
out. However, $M$ still has to be symmetric and positive definite and
this leads to the choice
\begin{equation}\label{eq:primal-dual-norm}
M =
\begin{bmatrix}
  \tfrac1\tau\id & -K^*\\
  -K  & \tfrac1\sigma\id
\end{bmatrix}.
\end{equation}
From~(\ref{eq:inclusion-iterate}) we get for $\lambda_k = 1$ the
following inertial primal-dual forward-backward algorithm
\begin{equation}
  \label{eq:ipdfb}
  \begin{cases}
    \xi^k  & = x^k + \alpha_k(x^k - x^{k-1})\\
    \zeta^k & = y^k + \alpha_k (y^k - y^{k-1})\\
    x^{k+1} & = (\id + \tau\partial G)^{-1}(\xi^k - \tau(\grad Q(\xi^k) + K^*\zeta^k))\\
    \bar \xi^{k+1} & = 2x^{k+1} - \xi^k\\
    y^{k+1} & =(\id + \sigma\partial F^*)^{-1}(\zeta^k - \sigma(\grad P^*(\zeta^k) - K\bar\xi^{k+1})).
  \end{cases}
\end{equation}
In the case that $Q= P^* = 0$ and $\alpha_k=0$ we obtain the primal-dual method
from~\cite{CP2010}.

The next two results characterize the conditions under which the
proposed inertial primal-dual forward-backward algorithm algorithm
converges.
\begin{theorem}
  \label{thm:convergence-saddle-point}
  The iterates given by method~(\ref{eq:ipdfb}) converge weakly to a
  solution of the saddle point problem~(\ref{eq:saddle-point-problem})
  if
  \begin{equation}\label{eq:ipdfb-conv1}
    \begin{split}
      0 < \tau < 2/L_Q, \quad 0 < \sigma < 2/L_P,\\
      \norm{K}^2 < (\tfrac1\tau - \tfrac{L_Q}{2})(\tfrac1\sigma - \tfrac{L_P}{2}),
    \end{split}
  \end{equation}
  and if $\alpha_k \in [0,\alpha]$ with $\alpha < 1$ and the iterates
  $(x^k,y^k)$ fulfill
  \begin{equation}\label{eq:ipdfb-conv2}
  \sum_{k=1}^\infty \alpha_k\norm{(x^k,y^k)-(x^{k-1},y^{k-1})}_M^2<\infty.
  \end{equation}
  Furthermore, condition~\eqref{eq:ipdfb-conv2} is fulfilled, if
  $\{\alpha_k\}$ is nondecreasing and there exists $\varepsilon > 0$
  such for all $\alpha_k$
  \begin{equation}\label{eq:ipdfb-conv3}
    \begin{split}
      \tfrac{1 - 3\alpha_k - \varepsilon}{\tau} \geq \tfrac{(1-\alpha_k)^2}{2} L_Q, \\
      \tfrac{1 - 3\alpha_k - \varepsilon}{\sigma} \geq \tfrac{(1-\alpha_k)^2}{2} L_P,\\
      \Big(\tfrac{1 - 3\alpha_k - \varepsilon}{\tau} -
      \tfrac{(1-\alpha_k)^2}{2} L_Q\Big) \Big( \tfrac{1 - 3\alpha_k -
        \varepsilon}{\sigma} - \tfrac{(1-\alpha_k)^2}{2} L_P\Big)\\ \geq
      (1-3\alpha_k-\varepsilon)^2\norm{K}^2\,.
    \end{split}
  \end{equation}
\end{theorem}
\begin{proof}
  Since $Q$ and $P^*$ are convex with Lipschitz-continuous gradients
  with Lipschitz constants $L_Q$ and $L_P$, respectively, it follows
  from the Baillon-Haddad Theorem (\cite[Corollary 18.16]{BauschkeCombettes})
  that that $\grad Q$ and $\grad P^*$ are
  co-coercive w.r.t. $L_Q^{-1}$ and $L_P^{-1}$, respectively. Hence,
  for $x,\xi\in X$ and $y,\zeta\in Y$ it holds that
  \begin{multline*}
    \scp{B(x,y)-B(\xi,\zeta)}{(x,y)-(\xi,\zeta)}_{X\times Y}  \\=
    \scp{\grad Q(x) -
      \grad Q(\xi)}{x-\xi}_X + \scp{\grad P^*(y) - \grad P^*(\zeta)}{y-\zeta}_Y\\
    \geq L_Q^{-1} \norm{\grad Q(x) - \grad Q(\xi)}_X^2 + L_P^{-1}\norm{\grad
      P^*(y) - \grad P^*(\zeta)}_Y^2.
  \end{multline*}
  Thus $B$ is co-coercive w.r.t. the mapping 
  \[
  L^{-1} =
  \begin{bmatrix}
    L_Q^{-1}\id & 0 \\
    0 & L_P^{-1}\id
  \end{bmatrix}
  \]
  It is easy to check that
  \[
  S = M - \tfrac12 L =
  \begin{bmatrix}
    (\tfrac1\tau - \tfrac{L_Q}{2})\id & -K^*\\
    -K & (\tfrac1\sigma - \tfrac{L_P}{2})\id
  \end{bmatrix}
  \]
  is positive definite if~\eqref{eq:ipdfb-conv1} is fulfilled
  and~\eqref{eq:ipdfb-conv2} follows from
  Theorem~\ref{thm:intertial_fb_convergence}.
  
  Applying condition~\eqref{eq:cond-alpha} to~\eqref{eq:ipdfb} we
  have:
  \[
  \begin{split}
    (1-3\alpha_k)
    \begin{bmatrix}
      \tfrac1\tau \id & -K^*\\
      -K & \tfrac1\sigma \id
    \end{bmatrix} - \frac{(1-\alpha_k)^2}{2}
    \begin{bmatrix}
      L_Q \id & 0\\
      0 & L_P\id
    \end{bmatrix}\\
    \geq \varepsilon
    \begin{bmatrix}
      \tfrac1\tau \id & -K^*\\
      -K & \tfrac1\sigma \id
    \end{bmatrix}
  \end{split}
  \]
  which can be checked to be true under the
  stated condition~\eqref{eq:ipdfb-conv3}.\qed
\end{proof}
Let us present some practical rules to choose feasible parameters for
the algorithm. For this we introduce the parameters $\gamma,\delta
\in(0,2)$, which can be interpreted as normalized step sizes in the
primal and dual variables and the parameter $r > 0$, which controls
the relative scaling between the primal and dual step sizes.
\begin{lemma}
  Choose $\gamma,\delta \in(0,2)$ and $r > 0$ and set
  \[
  \tau = \frac{1}{\norm{K}r + L_Q/\gamma} \quad \text{and} \quad \sigma =
  \frac{1}{\norm{K}/r + L_P/\delta}\,,
  \]
  Furthermore, let $\{\alpha_k\}$ be a non-decreasing sequence
  satisfying $0 \leq \alpha_k \leq \alpha(\gamma,\delta)$, where
  \begin{equation}
    \label{eq:alpha-bound}
    \alpha(\gamma,\delta) = 1 + \frac{\sqrt{9 - 4\max(\gamma,\delta) -
        2\varepsilon\max(\gamma,\delta)} - 3}{\max(\gamma,\delta)}\;,
  \end{equation}
  and $\varepsilon \in (0,(9 -
  4\max(\gamma,\delta))/(2\max(\gamma,\delta)))$. Then, the
  conditions~\eqref{eq:ipdfb-conv1} and~\eqref{eq:ipdfb-conv3} of
  Theorem~\ref{thm:convergence-saddle-point} hold,
  i.e. algorithm~\eqref{eq:ipdfb} converges weakly to a solution of
  the saddle-point problem~\eqref{eq:saddle-point-problem}.
\end{lemma}

\begin{proof}  
  It can be easily checked that the conditions~\eqref{eq:ipdfb-conv1}
  hold. Indeed, one has $\tau < 2/L_Q$, $\sigma < 2/L_P$ and also
  \begin{multline}\label{eq:estimate1}
  \left(\tfrac{1}{\tau} - \tfrac{L_Q}{2}\right)\left(\tfrac{1}{\sigma}
  - \tfrac{L_P}{2}\right) =\\ \norm{K}^2 +
  \norm{K}\left(\tfrac{rL_P(2-\delta)}{2\delta} +
  \tfrac{L_Q(2-\gamma)}{2\gamma r}\right) +
  L_QL_P\tfrac{(2-\gamma)(2-\delta)}{4\gamma\delta}\\ \geq \norm{K}^2\;.
  \end{multline}
  Next, we can compute the maximum value of $\alpha_k$ that ensures
  convergence of the algorithm. Observe that for any $r > 0$
  assertion~\eqref{eq:ipdfb-conv3} holds in particular if
  \[
  (1 - 3\alpha_k - \varepsilon) - \frac{(1-\alpha_k)^2}{2} \frac{\tau
    L_Q}{1-\tau\norm{K}r}\geq 0\,,
  \]
  and
  \[
  (1 - 3\alpha_k - \varepsilon) - \frac{(1-\alpha_k)^2}{2} \frac{\sigma
    L_P}{1-\sigma\norm{K}/r}\geq 0\,.
  \]
  where $\gamma=\frac{\tau L_Q}{1-\tau\norm{K}r}$ and
  $\delta=\frac{\sigma L_P}{1-\sigma\norm{K}/r}$. Clearly, the two
  inequalities are fulfilled if
  \[
  (1 - 3\alpha_k - \varepsilon) - \frac{(1-\alpha_k)^2}{2}
  \max(\gamma,\delta) \geq 0\,,
  \]
  from which the upper bound $\alpha(\gamma,\delta)$ follows.\qed
\end{proof}
\begin{remark}
  From equation~\eqref{eq:estimate1}, we can see that in case $L_P$ or
  $L_Q$ is zero (and fixed $\gamma$ respectively $\delta$), it might
  be favorable to choose larger respectively smaller values of $r$,
  since it leads to a smaller product of the terms on the left hand
  side of~\eqref{eq:estimate1} and hence to larger product of primal
  and dual step sizes.
\end{remark}
\subsection{Recovering known algorithms}
\label{sec:fb}

The proposed algorithm includes several popular algorithms for convex
optimization as special cases:
\begin{itemize}
\item \textbf{Forward-backward splitting:} Set $K=F^*=P^*=0$
  in~\eqref{eq:saddle-point-problem}, and set $M=\id$, $\lambda_k <
  2/L_Q$ and $\alpha_k=0$ in~\eqref{eq:ipdfb}. We obtain exactly the
  popular forward-backward splitting algorithm~\eqref{eq:fb-splitting}
  for minimizing the sum of a smooth and a non-smooth convex
  function. See~\cite{CombettesWajs,DaubechieISTA}.
\item \textbf{Nesterov's accelerated gradient method:} In addition to
  the previous setting, let $\lambda_k=1/L_Q$ and let the sequence
  $\{\alpha_k\}$ be computed according to one of the laws proposed
  in~\cite{Nesterov83,Nesterov,BeckTeboulle2008}. We can exactly
  recover Nesterov's accelerated gradient
  method~\cite{Nesterov83,Nesterov}, the accelerated proximal point
  algorithm~\cite{Gueler1991} and
  FISTA~\cite{BeckTeboulle2008}. 
  
  These algorithms offer an optimal convergence rate of
  $\mathcal{O}(1/k^2)$ for the function gap
  $(G+Q)(x^k)-(G+Q)(x^*)$. However, it is still unclear whether the
  sequence of iterates $\{x^k\}$ converges. We cannot give a full
  answer here but we can at least modify the FISTA algorithm such that
  the sequence $\alpha_k\|x^k-x^{k-1}\|^2$ has finite length.
  Following~\cite{AlvarezAttouch2001},
  condition~\eqref{it:assumpion-convergence} can be easily enforced
  ``on-line'' because it involves only past iterates. One possibility
  to ensure summability in~\eqref{it:assumpion-convergence} is to
  require that $\alpha_k\norm{x^k-x^{k-1}}_M^2 = \mathcal{O}(1/k^2)$,
  e.g.
  \begin{equation}\label{eq:saveguard}
    \alpha_k = \min((k-1)/(k+2), c/(k^2\norm{x^k-x^{k-1}}_M^2)),
  \end{equation}
  for some $c > 0$. However, since in the FISTA algorithm $\alpha_k =
  (k-1)/(k+2)\rightarrow 1$,
  Theorem~\ref{thm:intertial_fb_convergence} still does not imply
  convergence of the iterates. This is left for future investigation.
\item \textbf{Primal-dual algorithms:} Setting
  in~\eqref{eq:saddle-point-problem} $P^*=Q=0$ and let
  in~\eqref{eq:ipdfb} $\alpha_k=0$, we clearly obtain the first-order
  primal-dual algorithm investigated
  in~\cite{PBCC2009,Esseretal10,CP2010,HeYuan2012}. Furthermore, if we
  let $Q$ be a convex function with Lipschitz continuous gradient
  $\nabla Q$, we obtain the first-order primal-dual algorithm of
  Condat~\cite{Condat2013}. Moreover, in the present of smooth terms
  $Q$ and $P^*$ in the primal and dual problem, we recover V\~u's
  algorithm from~\cite{Vu2013}.  We point out that the methods
  in~\cite{Condat2013,Vu2013} involve an additional relaxation step of
  the form:
  \begin{equation}\label{eq:overrelaxation}
  x^{k+1} = ((1-\rho_k)\id + \rho_k(\id + \lambda_k T)^{-1})(x^k)\,,
  \end{equation}
  where $\rho_k$ is the relaxation parameter. In case there are no
  smooth term $Q$ and $P^*$ the relaxation parameter $\rho_k \in
  (0,2)$, in presence of $Q$ and $P^*$ the relaxation parameter is
  further restricted. See Section~\ref{sec:numerics} for numerical
  comparisons. Observe, that the overrelaxation technique is quite
  different from the inertial technique we used in this paper, which
  is of the form:
  \begin{equation}
  x^{k+1} = (\id + \lambda_k T)^{-1})(x^k + \alpha_k(x^k-x^{k-1}))\,.
  \end{equation}
  Indeed, it was shown in~\cite{Alvarez2003} that one can even use
  overrelaxation and inertial forces simultaneously. However,
  introducing an additional overrelaxation step in the proposed
  framework is left for future investigation.
\end{itemize}

\subsection{Preconditioning}
\label{sec:precond}

Besides the property of the map $M$ to make the primal-dual iterations
feasible, the map $M$ can also be interpreted as applying the
algorithm~\eqref{eq:inclusion-iterate} using $M=\id$ to the modified
inclusion:
\[
-M^{-1}B(x^*) \in M^{-1}A(x^*)\,,
\]
and hence, $M^{-1}$ can be interpreted as a left preconditioner to the
inclusion~\eqref{eq:ab-inclusion}. In the context of saddle point
problems, Pock and Chambolle~\cite{pock2011_precond} proposed a
preconditioning of the form
\[
M =
\begin{bmatrix}
  T^{-1} & -K^*\\
  -K & \Sigma^{-1}
\end{bmatrix}.
\]
where $T$ and $\Sigma$ are selfadjoint, positive definite maps. A
condition for the positive definiteness of $M$ follows from the
following lemma.
\begin{lemma}
  \label{lem:pos_def_block}
  Let $A_1$, $A_2$ be symmetric positive definite maps and $B$ a
  bounded operator. If $\norm{A_2^{-\frac{1}{2}}BA_1^{\frac{1}{2}}}<1$, then $A =
  \begin{bmatrix}
    A_1 & B^*\\
    B & A_2
  \end{bmatrix}
  $ is positive definite.  
\end{lemma}
\begin{proof}
  We calculate
  \[
  \scp{
    \begin{bmatrix}
      x\\y 
    \end{bmatrix}
  }
  {
    \begin{bmatrix}
      A_1 & B^*\\ B & A_2
    \end{bmatrix}
    \begin{bmatrix}
      x\\y
    \end{bmatrix}
  }
  = \scp{x}{A_1x} + 2\scp{Bx}{y} + \scp{y}{A_2 y}
  \]
  and estimate the middle term from below by Cauchy-Schwarz and Young's
  inequality and get for every $c>0$ that
  \begin{align*}
    \scp{Bx}{y} & = \scp{A_2^{-\frac{1}{2}}BA_1^{-\frac{1}{2}}A_1^{\frac{1}{2}} x}{A_2^{\frac{1}{2}}y}\\
    & \geq
    -\tfrac{c}2\norm{A_2^{-\frac{1}{2}}BA_1^{-\frac{1}{2}}}^2\norm{A_1^{\frac{1}{2}}x}^2 -
    \tfrac1{2c}\norm{A_2^{\frac{1}{2}}y}^2.
  \end{align*}
  Combining this with the assumption that
  $\norm{A_2^{-\frac{1}{2}}BA_1^{\frac{1}{2}}}<1$ we see that we can choose $c$ such
  that
  \[
  \begin{split}
    &\scp{
      \begin{bmatrix}
        x\\y
      \end{bmatrix}
    } {A
      \begin{bmatrix}
        x\\y
      \end{bmatrix}
    }\\
    &\geq (1 - c\norm{A_2^{-\frac{1}{2}}BA_1^{-\frac{1}{2}}}^2)\norm{A_1^{\frac{1}{2}}x}^2 +
    (1-\tfrac1c)\norm{A_2^{\frac{1}{2}}y}^2\\
    &>0
  \end{split}
  \]
  which proves the auxiliary statement.\qed
\end{proof}
We conclude that algorithm~(\ref{eq:inclusion-iterate}) converges as
long as one has $\norm{\Sigma^{-\frac{1}{2}}KT^{-\frac{1}{2}}} < 1$.  In order
to keep the proximal maps with respect to $G$ and $F^*$ feasible, the
maps $T$ and $\Sigma$ were restricted to diagonal matrices. However,
in recent work~\cite{BeckerFadili}, it was shown that some proximal
maps are still efficiently computable if $T$ and $\Sigma$ are the sum
of a diagonal matrix and a rank-one matrix.

Applying the preconditioning technique to the proposed inertial
primal-dual forward-backward algorithm~(\ref{eq:ipdfb}), we obtain the
method
\begin{equation}
  \label{eq:ipdfb-precond}
  \begin{cases}
    \xi^k  & = x^k + \alpha_k(x^k - x^{k-1})\\
    \zeta^k & = y^k + \alpha_k (y^k - y^{k-1})\\
    x^{k+1} & = (\id + T\partial G)^{-1}(\xi^k - T(\grad Q(\xi^k) + K^*\zeta^k))\\
    \bar\xi^{k+1} & = 2x^{k+1} - \xi^k\\
    y^{k+1} & =(\id + \Sigma\partial F^*)^{-1}(\zeta^k - \Sigma(\grad P^*(\zeta^k) - K\bar\xi^{k+1})).
  \end{cases}
\end{equation}
It turns out that the resulting method converges under appropriate
conditions.
\begin{theorem}
  \label{thm:convergence-saddle-point-preconditioning}
  In the setting of Theorem~\ref{thm:convergence-saddle-point} let
  furthermore $\nabla Q$ and $\nabla P^*$ be co-coercive w.r.t. the
  two bound, linear, symmetric and positive linear maps $D^{-1}$ and
  $E^{-1}$, respectively. If it holds that
  \begin{align}
    \Sigma^{-1} - \tfrac12 E &>0,\label{eq:sigma-E/2}\\
    T^{-1} - \tfrac12 D &> 0,\label{eq:T-D/2}\\
   \norm{(\Sigma^{-1} - \tfrac12 E)^{-\frac{1}{2}}K(T^{-1} - \tfrac12 D)^{-\frac{1}{2}}} &< 1\label{eq:pos-def1},
  \end{align}
  and that $\alpha_k$ is a nondecreasnig sequence in $[0,\alpha]$
  with $\alpha < 1$, and the iterates $(x^k,y^k)$
  of~(\ref{eq:ipdfb-precond}) fulfill
  \[
  \sum_{k=1}^\infty \alpha_k\norm{(x^k,y^k)-(x^{k-1},y^{k-1})}_M^2<\infty
  \]
  then $(x^k,y^k)$ converges weakly to a saddle point
  of~(\ref{eq:saddle-point-problem}). Furthermore, convergence is
  assured if there exists an $\varepsilon > 0$ such that for all
  $\alpha_k$ it holds that
  \begin{multline}
    \label{eq:pos-def2}
    (1-3\alpha_k-\varepsilon)\Sigma^{-1} \geq \tfrac{(1-\alpha_k)^2}{2}
    E,\\
    (1-3\alpha_k-\varepsilon)T^{-1} \geq \tfrac{(1-\alpha_k)^2}{2} D,\\
    \Big\|\Big((1-3\alpha_k-\varepsilon)\Sigma^{-1} -
    \tfrac{(1-\alpha_k)^2}{2}
    E\Big)^{-\frac{1}{2}}K\\\Big((1-3\alpha_k-\varepsilon)T^{-1} -
    \tfrac{(1-\alpha_k)^2}{2} D\Big)^{-\frac{1}{2}}\Big\| \leq \tfrac{1}{(1-3\alpha_k-\varepsilon)}\;.
  \end{multline}
\end{theorem}
\begin{proof}
  We start by setting
  \[
  C =
  \begin{bmatrix}
    D & 0\\
    0 & E
  \end{bmatrix}.
  \]
  and by Theorem~\ref{thm:intertial_fb_convergence} we only need to check if
  $S = M - \tfrac12 C$ is positive.
  Obviously, the diagonal blocks of $S$ are positive,
  by (\ref{eq:sigma-E/2}) and~(\ref{eq:T-D/2}).
  
  Now we use Lemma~\ref{lem:pos_def_block} to see
  that~(\ref{eq:sigma-E/2}),~(\ref{eq:T-D/2}) and~(\ref{eq:pos-def1})
  imply that $S$ is positive definite.  For the second claim, we
  employ Theorem~\ref{thm:intertial_fb_alphas} and only need to show
  that $R = (1 - 3\alpha_k)M - \tfrac{(1-\alpha_k)^2}{2}C\geq\epsilon M$
  which is equivalent to showing that
  \[
  (1-3\alpha_k-\epsilon)
  \begin{bmatrix}
    T^{-1} & -K^*\\
    -K & \Sigma^{-1}
  \end{bmatrix}
   - \tfrac{(1-\alpha_k)^2}{2}
   \begin{bmatrix}
     D & 0\\ 0 & E
   \end{bmatrix}
   \geq 0.
  \]
  Again using Lemma~\ref{lem:pos_def_block} we obtain that~(\ref{eq:pos-def2}) ensures this.\qed
\end{proof}

\subsection{Diagonal Preconditioning}
\label{sec:diagprecond}
In this section, we show how we can choose pointwise step sizes for
both the primal and the dual variables that will ensure the
convergence of the algorithm. The next result is an adaption of the
preconditioner proposed in~\cite{pock2011_precond}.

\begin{lemma} \label{lemma2} Assume that $\nabla Q$ and $\nabla P^*$
  are co-coercive with respect to diagonal matrices $D^{-1}$ and
  $E^{-1}$, where $D=\textup{diag}(d_1, \ldots, d_n)$ and
  $E=\textup{diag}(e_1, \ldots, e_n)$. Fix $\gamma,\delta \in (0,2)$, $r >
  0$, $s \in [0,2]$ and let $T = \textup{diag}(\tau_1,...\tau_n)$
  and $\Sigma = \textup{diag}(\sigma_1,...,\sigma_m)$ with
  \begin{equation}\label{eq:choice-tau-sigma}
    \tau_j = \frac{1}{\frac{d_j}{\gamma} + r\sum_{i=1}^m
      |K_{i,j}|^{2-s}}\;, \quad
    \sigma_i = \frac{1}{\frac{e_i}{\delta} + \tfrac1 r\sum_{j=1}^n |K_{i,j}|^{s}}
  \end{equation}
then it holds that
\begin{equation}\label{eq:diff-pos}
  \Sigma^{-1}-\tfrac{1}{2} E > 0\,, \quad T^{-1}-\tfrac{1}{2} D > 0\;,
\end{equation}
\begin{equation}\label{eq:ineq-precond}
  \norm{(\Sigma^{-1}-\tfrac{1}{2}E)^{-\frac{1}{2}}K (T^{-1}-\tfrac{1}{2}D)^{-\frac{1}{2}}} \leq 1\;.
\end{equation}
Furthermore, equation~(\ref{eq:pos-def2}) is fulfilled
if~(\ref{eq:alpha-bound}) is fulfilled.
\end{lemma}
\begin{proof}
  The first two conditions follow from the fact that for diagonal
  matrices, the~\eqref{eq:diff-pos} can be written pointwise. By the
  definition of $\tau_j$, and $\sigma_i$ it follows that for any $s
  \in [0,2]$ and using the convention that $0^0 = 0$,
  \[
  \frac{1}{\tau_i}-\frac{d_i}{2} >
  \frac{1}{\tau_i}-\frac{d_i}{\gamma} = r\sum_{i=1}^m
  |K_{i,j}|^{2-s} \geq 0 \,,
  \]
  and
  \[
  \frac{1}{\sigma_i}-\frac{e_i}{2} >
  \frac{1}{\sigma_i}-\frac{e_i}{\delta} = \tfrac1 r\sum_{j=1}^n
  |K_{i,j}|^{s} \geq 0 \,.
  \]
  For the third condition, we have that for any $s \in [0,2]$
  \begin{align}
    &\norm{(\Sigma^{-1}-\tfrac{1}{2}E)^{-\frac{1}{2}}K (T^{-1}-\tfrac{1}{2}D)^{-\frac{1}{2}}x}^2\nonumber\\
    &=
    \sum_{i=1}^m \left(\sum_{j=1}^n  \frac{1}{\sqrt{\tfrac{1}{\sigma_i}-\tfrac{e_i}{2}}}K_{i,j}\frac{1}{\sqrt{\tfrac{1}{\tau_j}-\tfrac{d_j}{2}}}x_j\right)^2\nonumber\\
    &=\sum_{i=1}^m \frac{1}{\tfrac{1}{\sigma_i}-\tfrac{e_i}{2}} \left(\sum_{j=1}^n K_{i,j}\frac{1}{\sqrt{\tfrac{1}{\tau_j}-\tfrac{d_j}{2}}}x_j\right)^2\nonumber\\
    &<\sum_{i=1}^m \frac{1}{\tfrac{1}{\sigma_i}-\tfrac{e_i}{\delta}} \left(\sum_{j=1}^n |K_{i,j}|^{\frac{s}{2}}|K_{i,j}|^{1-\frac{s}{2}} \frac{1}{\sqrt{\tfrac{1}{\tau_j}-\tfrac{d_j}{\gamma}}} x_j\right)^2\nonumber\\
    &\leq \sum_{i=1}^m \frac{1}{\tfrac{1}{\sigma_i}-\tfrac{e_i}{\delta}} \left(\sum_{j=1}^n |K_{i,j}|^{s}\right)\left( \sum_{j=1}^n
      |K_{i,j}|^{2-s} \frac{1}{\tfrac{1}{\tau_j}-\tfrac{d_j}{\gamma}}  x_j^2\right)\;,\label{eq:est-ineq-precond}
  \end{align}
  where the second line follows from $K_{i,j} \leq |K_{i,j}|$ and
  $\gamma,\delta < 2$ and the last line follows from the
  Cauchy-Schwarz inequality.  By definition of $\sigma_i$ and
  $\tau_j$, and introducing $r > 0$, the above estimate can be
  simplified to
  \begin{align}
    &\sum_{i=1}^m \frac{1/r}{\tfrac{1}{\sigma_i}-\tfrac{e_i}{\delta}} \left(\sum_{j=1}^n |K_{i,j}|^{s}\right)\left( \sum_{j=1}^n
      |K_{i,j}|^{2-s} \frac{r}{\tfrac{1}{\tau_j}-\tfrac{d_j}{\gamma}}  x_j^2\right) \nonumber\\
    &= \sum_{i=1}^m \sum_{j=1}^n |K_{i,j}|^{2-s}\frac{r}{\tfrac{1}{\tau_j}-\tfrac{d_j}{\gamma}}  x_j^2\nonumber\\
    &=  \sum_{j=1}^n \left(\sum_{i=1}^m |K_{i,j}|^{2-s}\right) \frac{r}{\tfrac{1}{\tau_j}-\tfrac{d_j}{\gamma}} x_j^2 = \|x\|^2\;.\label{eq:est2-ineq-precond}
  \end{align}
  Using the above estimate in the definition of the operator norm, we
  obtain the desired result
  \begin{multline}
    \norm{(\Sigma^{-1}-\tfrac{1}{2}E)^{-\frac{1}{2}}K
      (T^{-1}-\tfrac{1}{2}D)^{-\frac{1}{2}}}^2\\
    = \sup_{x \not=
      0}\frac{\norm{(\Sigma^{-1}-\tfrac{1}{2}E)^{-\frac{1}{2}}K
        (T^{-1}-\tfrac{1}{2}D)^{-\frac{1}{2}}x}^2}{\|x\|^2} \leq 1\;.
  \end{multline}
  If we now assume that~(\ref{eq:alpha-bound}) is fulfilled, we
  especially obtain that
  \[
  \frac{1-3\alpha_k - \epsilon}\delta\geq \frac{(1-\alpha_k)^2}2\quad\text{and}\quad  \frac{1-3\alpha_k - \epsilon}\gamma\geq \frac{(1-\alpha_k)^2}2
  \]
  and consequently, by using the definition of $\tau_j$ and $\sigma_i$
  from~(\ref{eq:choice-tau-sigma}), that
  \begin{align*}
    \frac{1/r}{\tfrac{1-3\alpha_k-\epsilon}{\sigma_i} -
      \tfrac{(1-\alpha_k)^2}{2}e_i}\sum_j\abs{K_{ij}}^s
    & \leq\frac{1}{1-3\alpha_k-\epsilon}\\
    \text{and}\quad\frac{r}{\tfrac{1-3\alpha_k-\epsilon}{\tau_i}
      - \tfrac{(1-\alpha_k)^2}{2}d_i}\sum_i\abs{K_{ij}}^{2-s}
    & \leq\frac{1}{1-3\alpha_k-\epsilon}.
  \end{align*}
  Now one can use the same arguments as in
  inequalities~(\ref{eq:est-ineq-precond})
  and~(\ref{eq:est2-ineq-precond}) to derive that~(\ref{eq:pos-def2})
  is fulfilled.\qed
\end{proof}

\section{Numerical experiments}
\label{sec:numerics}

In this section, we provide several numerical results based on simple
convex image processing problems to investigate the numerical
properties of the proposed algorithm.

\subsection{TV-$\ell_2$ denoising}

Let us investigate the well-known total variation denoising model:
\begin{equation}\label{eq:rof}
\min_u \|\nabla u\|_{2,1} + \frac{\lambda}{2} \|u-f\|_2^2\;,
\end{equation}
where $f \in \R^{MN}$ is a given noisy image of size $M\times N$
pixels, $u\in\R^{MN}$ is the restored image, $\nabla \in \R^{2MN\times
  MN}$ is a sparse matrix implementing the discretized image gradient
based on simple forward differences. The operator norm of $\nabla$ is
computed as $\sqrt{8}$. The parameter $\lambda > 0$ is used to control
the trade-off between smoothness and data fidelity. For more
information we refer to~\cite{CP2010}. Figure~\ref{fig:rof} shows an
exemplary denoising result, where we used the noisy image on the left
hand side as input image and set $\lambda=10$.

\begin{figure*}[hbt]
  \begin{center}
  \subfigure[Noisy image]{\includegraphics[width=0.49\textwidth]{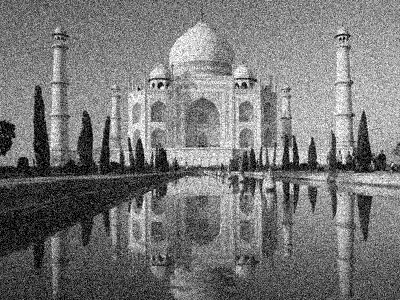}}\hfill
  \subfigure[Restored image]{\includegraphics[width=0.49\textwidth]{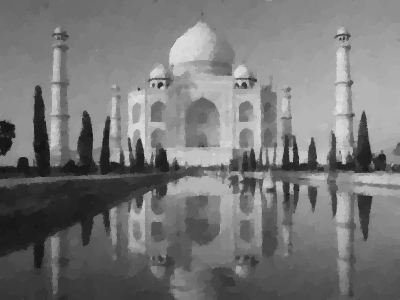}}
    \end{center}
    \caption{Application to total variation based image denoising with
      $\ell_2$ fitting term. (a) shows the noisy input image
      containing Gaussian noise with a standard deviation of
      $\sigma=0.1$, (b) shows the restored image using $\lambda=10$.}\label{fig:rof}
\end{figure*}
The dual problem associated to~\eqref{eq:rof} is given by the
following optimization problem
\begin{equation}\label{eq:rof-dual}
\min_p \frac{1}{2}\|\lambda f-\nabla^Tp\|_2^2 + I_P(p)\;,
\end{equation}
where $p \in \R^{2MN}$ is the dual variable and $I_P$ is the indicator
function for the set $P = \{p\in \R^{2MN}: \|p\|_{2,\infty} \leq
1\}$. This problem can easily cast into the problem
class~\eqref{eq:saddle-point-problem}, by setting $Q(p) =
\frac{1}{2}\|\lambda f - \nabla^Tp\|^2_2$, $G=I_P(p)$, and
$K=F^*=P^*=0$.

In our first experiment of this section, we study the behavior of the
error $e_k=\alpha_k\|x_k-x_{k-1}\|_2^2$, which plays a central role in
showing convergence of the algorithm. Figure~\ref{fig:fista-error}
shows the convergence of the sequence $\{e_k\}$ generated by the FISTA
algorithm by additionally using~\eqref{eq:saveguard} for different
choices of the constant $c$. The left figure depicts a case where $c$
is not chosen large enough and hence the save guard shrinks the
extrapolation factor $\alpha_k$ such that the error $e_k$ still
converges with rate $1/k^2$. The right figure shows a case where $c$
is chosen sufficiently large and hence the save guard does not
apply. In this case, the algorithm produces the same sequence of
iterates as the original FISTA algorithm. From our numerical results,
it seems that the asymptotic convergence of $e_k$ is actually faster
that $1/k^2$ which suggest that the iterates of FISTA are indeed
convergent.

\begin{figure*}[ht!]
  \begin{center}
  \subfigure[$c=10^4$]{\includegraphics[width=0.4\textwidth]{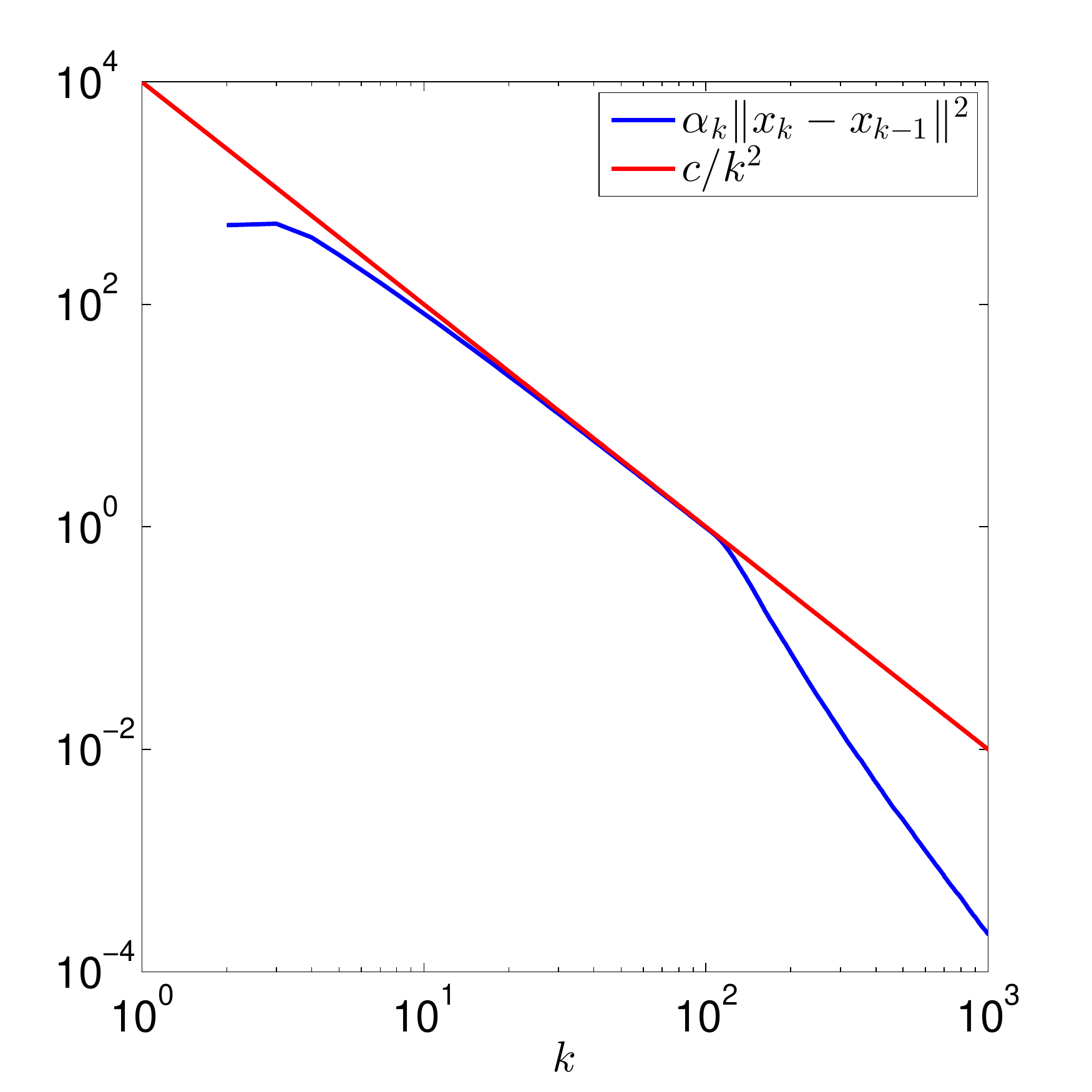}}\hfill
  \subfigure[$c=10^5$]{\includegraphics[width=0.4\textwidth]{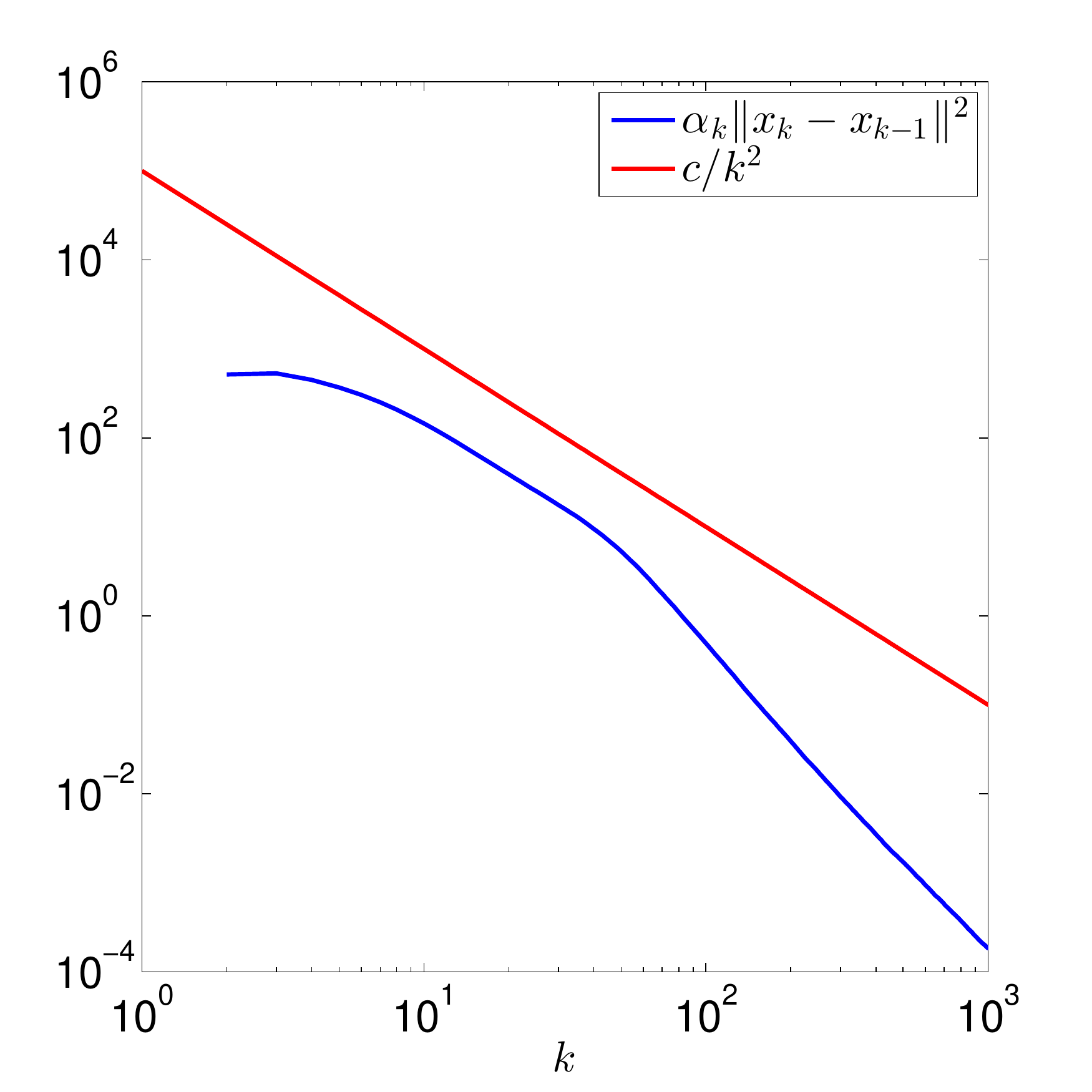}}
\end{center}
\caption{Convergence of the error sequence in the FISTA
  algorithm.}\label{fig:fista-error}
\end{figure*}

In the second experiment we consider a saddle-point formulation
of~\eqref{eq:rof}
\begin{equation}\label{eq:rof-saddle}
  \min_u\max_p \, \scp{\nabla u}{p} + \frac{\lambda}{2} \|u-f\|_2^2 - I_P(p)\;,
\end{equation}
Casting this problem in the general
from~\eqref{eq:saddle-point-problem}, the most simple choice is $K =
\nabla$, $F^*(p) = I_P(p)$, $G(u) = \|u-f\|_2^2$, $Q=P^*=0$. Hence,
algorithm~\ref{eq:ipdfb} reduces to an inertial variant of the
primal-dual algorithm of~\cite{CP2010}. According
to~\eqref{eq:ipdfb-conv1} the step sizes $\tau$ and $\sigma$ need to
satisfy $\tau\sigma < 1/\|K\|^2$, but the ratio $\tau/\sigma$ can be
chosen arbitrarily.

Figure~\ref{fig:primal-dual} shows the convergence of the primal dual
gap for different choices for $\alpha_k$ and the ratio
$\tau/\sigma$. In general, one can see that the convergence becomes
faster for larger values of $\alpha_k$. According
to~\eqref{eq:ipdfb-conv3}, we can guarantee convergence for $\alpha_k
< 1/3$ but we cannot guarantee convergence for larger values of
$\alpha_k$. In fact, it turns out that the feasible range of
$\alpha_k$ depends on the ratio $\tau/\sigma$. For $\tau/\sigma=0.1$,
fastest convergence is obtained by choosing $\alpha_k$ dynamically as
$\alpha_k=(k-1)/(k+2) \rightarrow 1$. In this case, the primal-dual
shows a very similar performance to the FISTA algorithm. For
$\tau/\sigma=0.01$, the algorithm converges for up to $\alpha_k=1/2$,
but diverges for the dynamic choice. This behavior can be explained by
the fact that the ratio $\tau/\sigma$ directly influences the
$M$-metric~\eqref{eq:primal-dual-norm} which in turn leads to a
divergence of the error term $\sum_{k=1}^\infty
\alpha_k\norm{x^k-x^{k-1}}_M^2$.

Next, we provide an experiment, where we compare the effect of the
inertial force with the effect of overrelaxation that has already been
considered
in~\cite{Condat2013,Vu2013}. Figure~\ref{fig:inertial_vs_overrelaxation}
compares the primal-dual gap of the plain primal-dual
(i.e. $\alpha_k=0$) algorithm~\cite{CP2010} with the performance of
its variants using either inertial forces using $\alpha_k=1/2$ or
overrelaxation (see~\eqref{eq:overrelaxation}) using $\rho_k=1.9$. For
all methods we used $\tau/\sigma=0.01$. Both variants improve the
convergence of the plain primal-dual algorithm but we observed that
overrelaxation leads to some numerical oscillations, in particular for
values of $\rho_k$ close to $2$.

\begin{figure*}[ht!]
  \begin{center}
  \subfigure[$\tau/\sigma=0.1$]{\includegraphics[width=0.4\textwidth]{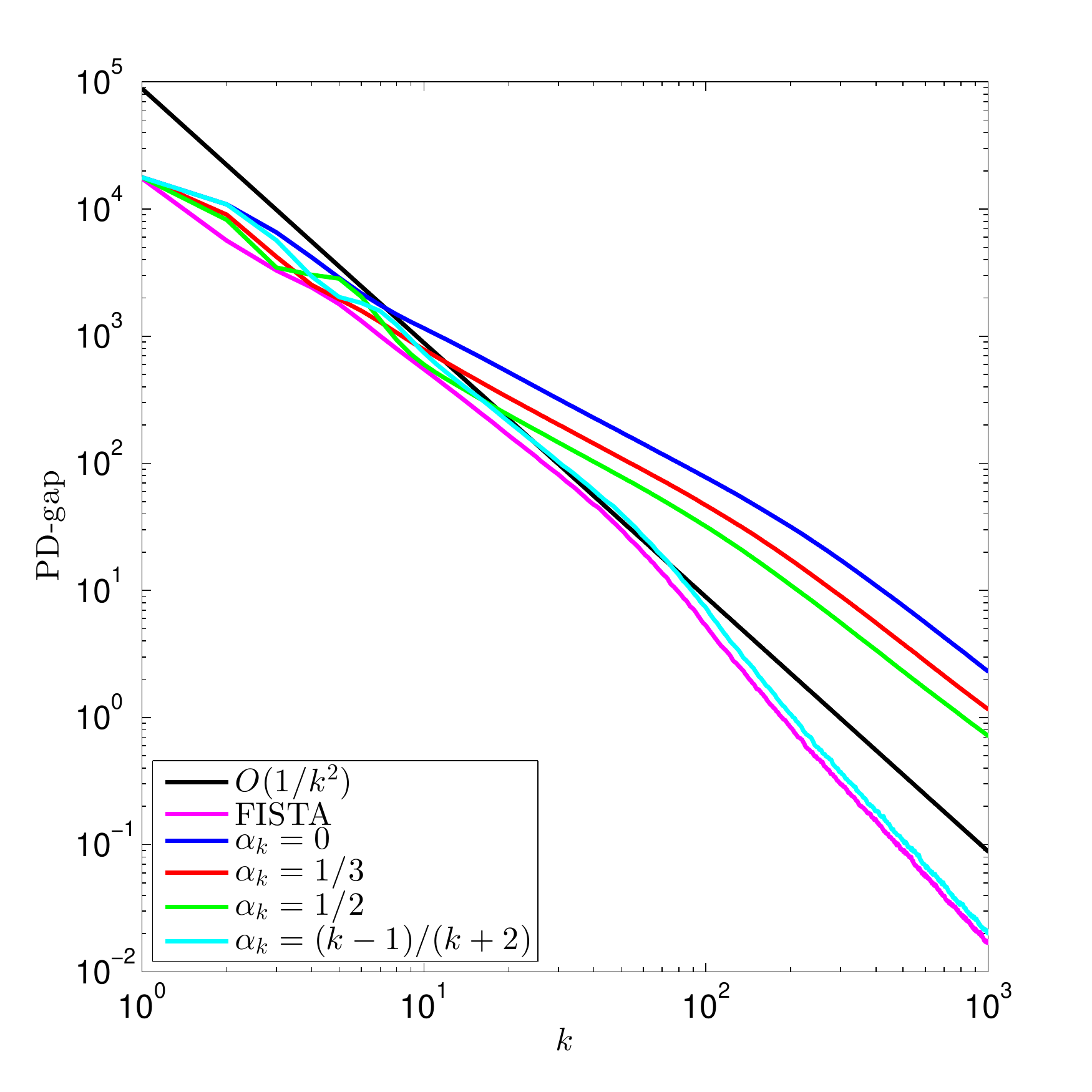}}\hfill
  \subfigure[$\tau/\sigma=0.01$]{\includegraphics[width=0.4\textwidth]{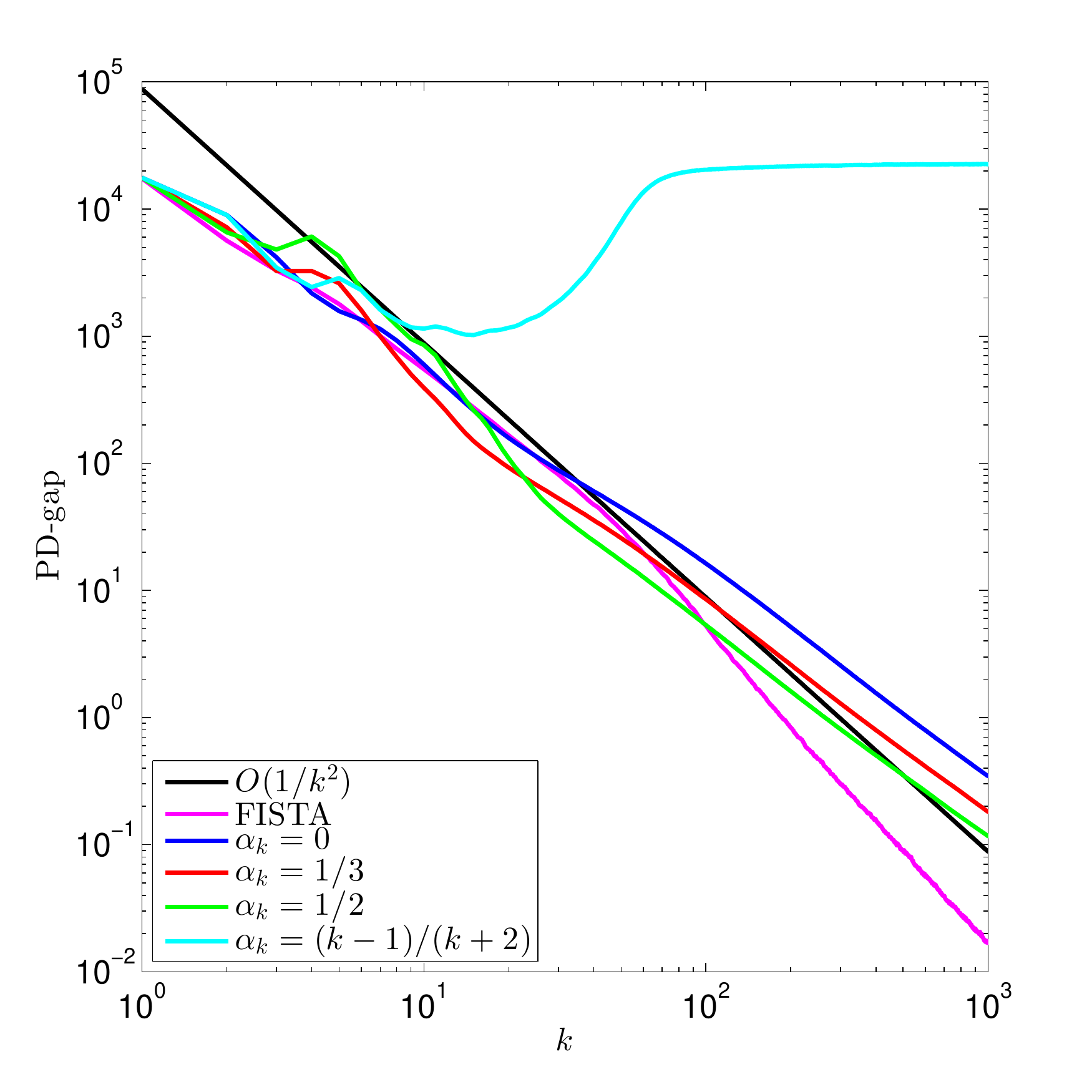}}
\end{center}
\caption{Convergence of the inertial primal-dual forward-backward
  algorithm~\eqref{eq:ipdfb} for different choices of
  $\tau/\sigma$ and $\alpha_k$.}\label{fig:primal-dual}
\end{figure*}

\begin{figure}
  \begin{center}
    \includegraphics[width=0.4\textwidth]{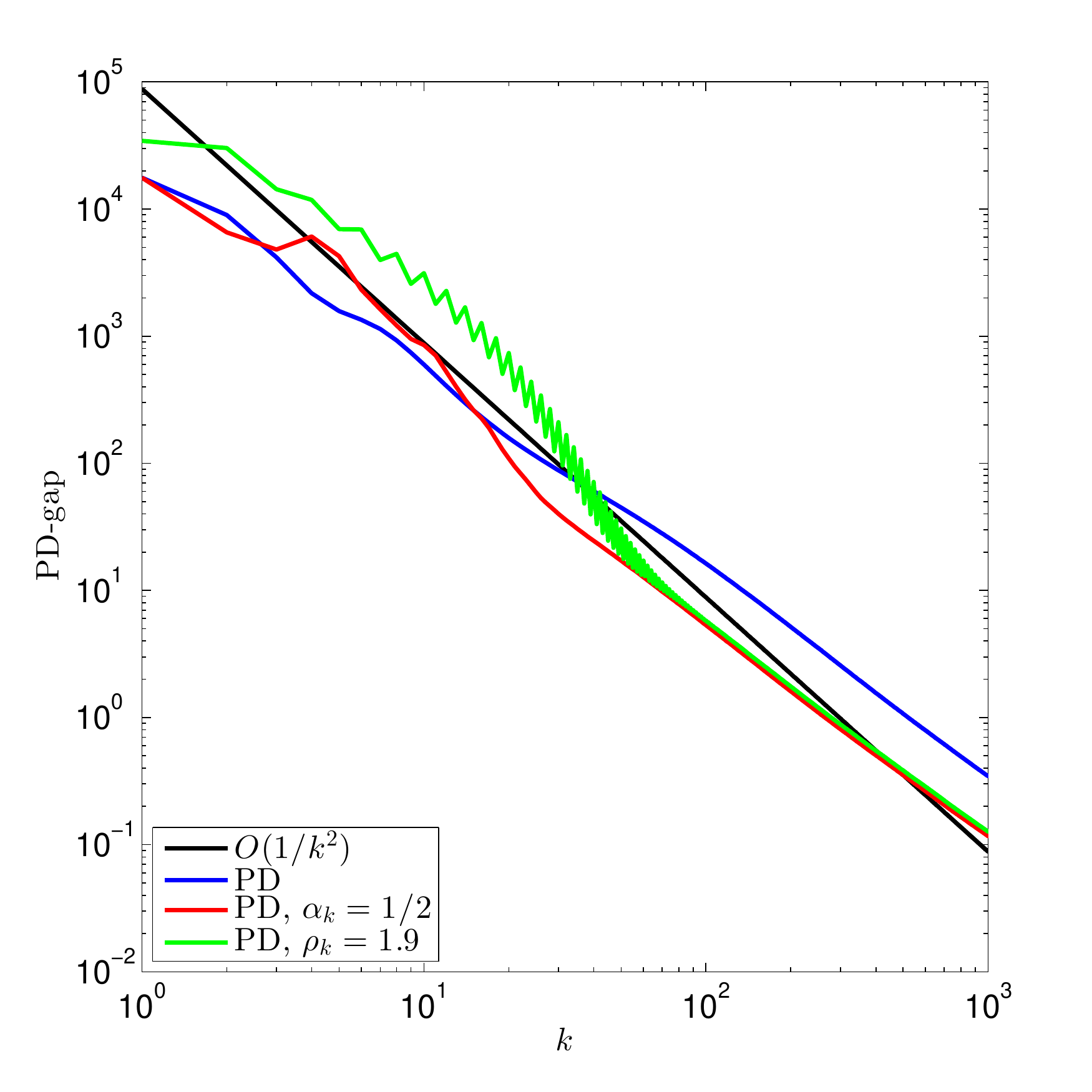}
  \end{center}
  \caption{Comparison between inertial forces and overrelaxation. Both
    techniques show similar performance improvement but overrelaxation
    appears numerical less stable.}\label{fig:inertial_vs_overrelaxation}
\end{figure}

\subsection{TV-$\ell_2$ deconvolution}

Our next example incorporates an additional linear operator
in the data fidelity. The problem is given by
\begin{equation}\label{eq:rof-deconv}
\min_u \|\nabla u\|_{2,1} + \frac{\lambda}{2} \|Hu-f\|_2^2\;,
\end{equation}
where $H \in \R^{MN\times MN}$ is a linear operator, for example $H$
can be such that $Hu$ is equivalent to the 2D convolution $h * u$,
where $h$ is a convolution kernel.
\begin{figure*}
  \begin{center}
  \subfigure[Noisy and blurry image]{\includegraphics[width=0.49\textwidth]{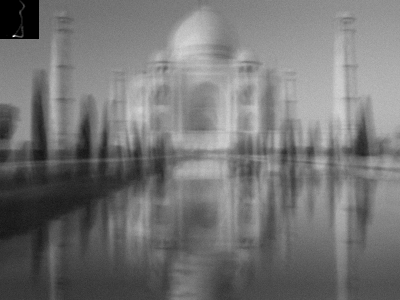}}\hfill
  \subfigure[Restored image]{\includegraphics[width=0.49\textwidth]{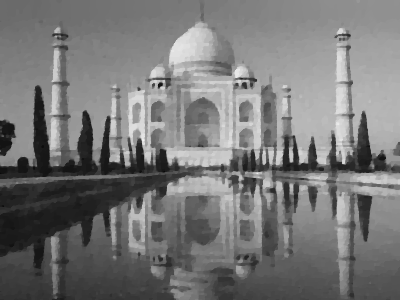}}
    \end{center}
    \caption{Application to total variation based image deconvolution
      with $\ell_2$ fitting term. (a) shows the noisy ($\sigma=0.01)$
      and blurry input image together with the known point spread
      function, and (b) shows the restored image using
      $\lambda=1000$.}\label{fig:deconv_l2}
\end{figure*}
We again consider a saddle-point formulation
\begin{equation}\label{eq:rof-deconv-saddle}
  \min_u\max_p \, \scp{\nabla u}{p} + \frac{\lambda}{2} \|Hu-f\|_2^2 - I_P(p)\;.
\end{equation}
Casting this problem into the general class of
problems~\eqref{eq:saddle-point-problem}, one has different
possibilities. If we would choose $K=\nabla$, $F^*(p)=I_P(p)$,
$G(u)=\frac{\lambda}{2} \|Hu-f\|_2^2$, $Q,P^*=0$, we would have to
compute the proximal map with respect to $G$ in each iteration of the
algorithm, which can be computationally very expensive. Instead, if we
choose $G=0$, $Q(u)=\frac{\lambda}{2} \|Hu-f\|_2^2$, we only need to
compute $\nabla Q(u)= \lambda H^T(Hu-f)$ which is obviously much
cheaper. We call this variant the \textit{explicit}
variant. Alternatively, we can additionally dualize the data term,
which leads to the extended saddle-point problem
\[
\min_u\max_{p,q} \, \scp{\nabla u}{p} + \scp{Hu}{q} + \scp{f}{q} -
\frac{1}{2\lambda} \|q\|_2^2 - I_P(p)\;.
\]
where $q \in \R^{MN}$ is the new dual variable vector.  Casting now
this problem into~\eqref{eq:saddle-point-problem}, we identify
$K=\begin{pmatrix} \nabla\\ H \end{pmatrix}$, $G=0$, $Q=0$, $F^*(p,q)
= I_P(p) + \frac{1}{2\lambda}\|q\|_2^2 - \scp{f}{q}$, which eventually
leads to proximal maps that are easy to compute. We call this variant the
\textit{split-dual} variant.

Figure~\ref{fig:deconv_l2_comparison} shows a comparison of the
convergence between the explicit and the split-dual variants with and
without inertial forces. For the explicit variant, the maximal value
of the inertial force was computed using
formula~\eqref{eq:alpha-bound}, where we set $L_K=\sqrt{8}$,
$L_Q=\lambda$, $\gamma=1$ and $r=100$. This results in a theoretically
maximal value of $\alpha_k=0.236$ but the algorithm also
converges for $\alpha_k=1/3$ (see Remark~\ref{remark:implicit}). For the
split-dual variant, the formulation does not involve any explicit
terms and hence the maximal feasible value for $\alpha_k$ is
$1/3$. The primal and dual step sizes were computed according to the
preconditioning rules~\eqref{eq:choice-tau-sigma} (skipping the
explicit terms), where we again used $r=100$.

The figure shows the primal energy gap, where the optimal primal
energy value has been computed by running the explicit variant for
$10000$ iterations. The algorithms were stopped, after the primal
energy-gap was below a threshold of $10^{-2}$. From the figure, one
can see that for both variants, the inertial force leads to a faster
convergence. One can also see that in the early stage of the
iterations, the explicit variant seems to converge faster than the
split-dual variant. Finally, we point out that the asymptotic
convergence of both variants is considerably faster than
$\mathcal{O}(1/k^2)$.

\begin{figure}
  \begin{center}
    \includegraphics[width=0.45\textwidth]{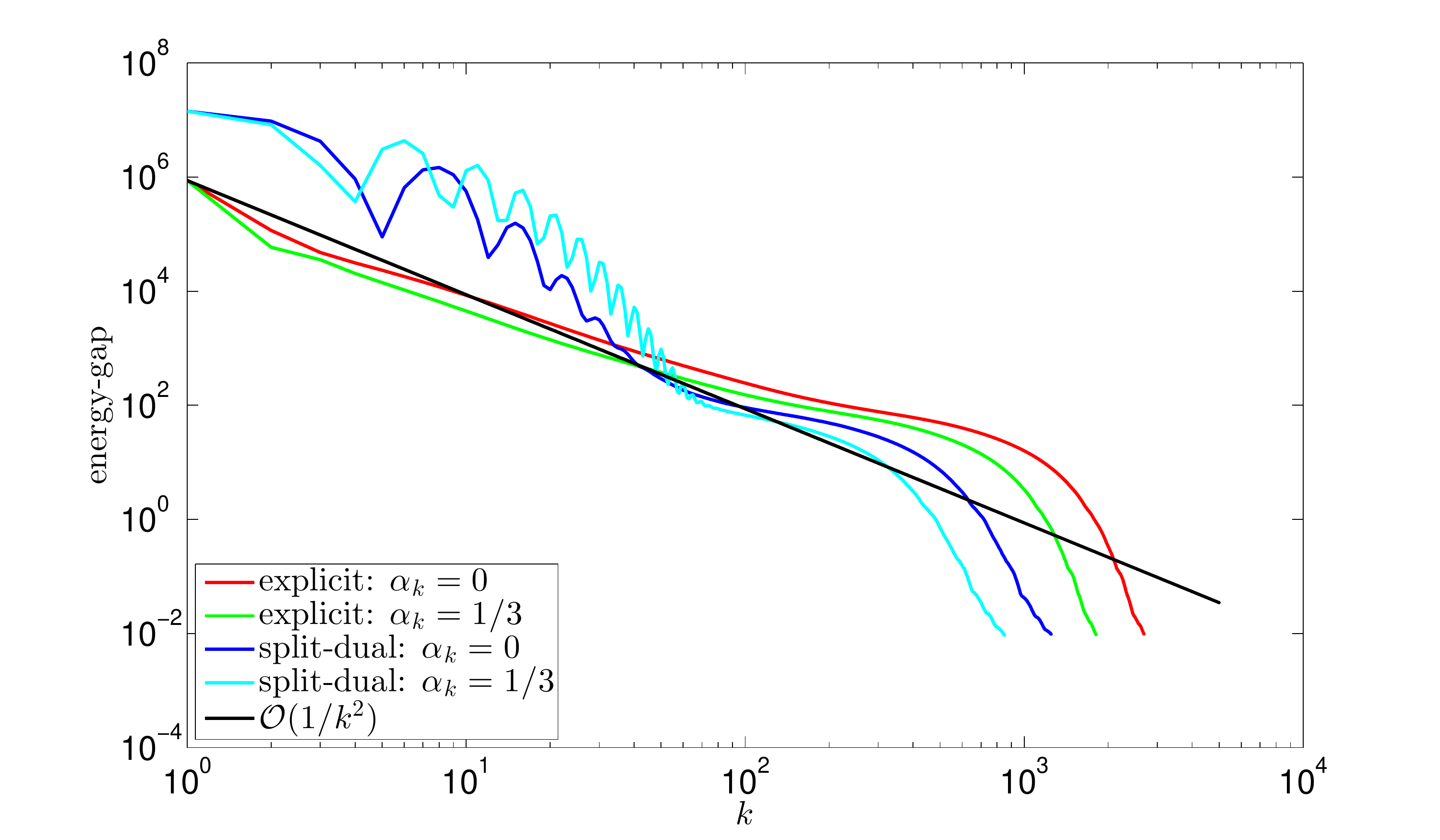}
  \end{center}
  \caption{Convergence of the primal dual algorithms with and without
    inertial forces.}\label{fig:deconv_l2_comparison}
\end{figure}

\section{Conclusion}
\label{sec:conclusion}

In this paper we considered an inertial forward-backward algorithm for
solving monotone inclusions given by the sum of a monotone operator
with an easy-to-compute resolvent operator and another monotone
operator which is co-coercive. We have proven convergence of the
algorithm in a general Hilbert space setting. It turns out that the
proposed algorithm generalizes several recently proposed algorithms
for example the FISTA algorithm of Beck and
Teboulle~\cite{BeckTeboulle2008} and the primal-dual algorithm of
Chambolle and Pock~\cite{CP2010}. This gives rise to new inertial
primal-dual algorithms for convex-concave programming. In several
numerical experiments we demonstrated that the inertial term leads to
faster convergence while keeping the complexity of each iteration
basically unchanged.

Future work will mainly concentrate on trying to find worst-case
convergence rates for particular problem classes.

\section*{Acknowledgements}
Thomas Pock acknowledges support from the Austrian science fund (FWF)
under the project "Efficient algorithms for nonsmooth optimization in
imaging", No. I1148 and the FWF-START project Bilevel optimization for
Computer Vision, No. Y729. The authors wish to thank Antonin Chambolle
for very helpful discussions.

\bibliographystyle{plain}
\bibliography{iPDFB}

\end{document}